\documentclass{article}

\usepackage{dependencies}
\usepackage{placeins}
\usepackage{booktabs}
\usepackage{authblk}

\title{Deep Diffusion Maps}
\author[1]{Sergio García-Heredia}
\author[1]{Carlos M. Alaíz}
\author[1]{Ángela Fernández}
\affil[1]{Computer Science Department, Universidad Autónoma de Madrid}
\date{\today}

\begin{document}

\maketitle

\section*{Abstract}

One of the fundamental problems within the field of machine learning is dimensionality reduction. Dimensionality reduction methods make it possible to combat the so-called curse of dimensionality, visualize high-dimensional data and, in general, improve the efficiency of storing and processing large data sets. One of the best-known nonlinear dimensionality reduction methods is Diffusion Maps. However, despite their virtues, both Diffusion Maps and many other manifold learning methods based on the spectral decomposition of kernel matrices have drawbacks such as the inability to apply them to data outside the initial set, their computational complexity, and high memory costs for large data sets. In this work, we propose to alleviate these problems by resorting to deep learning. Specifically, a new formulation of Diffusion Maps embedding is offered as a solution to a certain unconstrained minimization problem and, based on it, a cost function to train a neural network which computes Diffusion Maps embedding ---both inside and outside the training sample--- without the need to perform any spectral decomposition. The capabilities of this approach are compared on different data sets, both real and synthetic, with those of Diffusion Maps and the Nystr\"om method.

\textbf{Keywords:} dimensionality reduction, manifold learning, Diffusion Maps, deep learning, neural networks, Deep Diffusion Maps.

\section{Introduction}\label{sec:introduction}
Machine Learning (ML) is a branch of artificial intelligence focused on the development of algorithms that allow computers to learn, that is, improve their performance in a specific task, based on data. This capability has led ML to become one of the most important and active areas of technological research and development nowadays. It is undoubtedly a key tool for harnessing the vast amount of data generated and collected every day in the era of big data. However, within this field, one of the key challenges is the handling of high-dimensional samples, where algorithms face difficulties due to the so-called curse of dimensionality.

The increase prevalence of high-dimensional data has led to the development of an entire area within ML called Dimensionality Reduction (DR). Its objective is to find a new representation of the data, called embedding, with fewer variables or features than the original representation, while trying not to lose information relevant to the analysis or the task at hand. This is a challenge that has been addressed in different ways in the literature. These techniques, in addition to facilitating the work of other ML algorithms, have additional applications, such as allowing the visualization and better interpretation of data or reducing its storage and subsequent processing costs.

A popular DR method, framed within the so-called manifold learning methods, is Diffusion Maps (DM;~\cite{coifman_diffusion_2006}). This approach is inspired by the physical phenomenon of diffusion. It starts by constructing an affinity matrix based on the local distances between points and usually defined by a kernel function, which allows to capture the geometric structure of the underlying manifold of the data. Then, a random walk is defined over a graph whose nodes are the data points and whose weights are given by the kernel matrix. From this random walk, a distance on the graph, namely the diffusion distance, is defined, together with a representation of the data capable of preserving, to some extent, such distance. This ability to preserve local relationships between points is one of the fundamental differences between DM and other, so-called global, methods such as PCA~\cite{pearson_liii_1901}.

Despite their advantages, both DM and many other manifold learning methods based on the spectral decomposition of kernel matrices present significant challenges. These include their high computational complexity and high memory costs, especially when dealing with large data sets. Furthermore, the application of these methods to new points outside the initial data sample (out-of-sample points) is problematic, since, in principle, it requires applying the complete algorithm again on the data set extended to the new points. This is unfeasible when the data set is very large or when working with a constant flow of new data.

To address these problems, there are methods that approximate the decomposition of the kernel matrices based on a representative subsample of the data, such as the well-known Nystr\"om method~\cite{williams_using_2000, bengio_out--sample_2003}. A more recent alternative consists in using Deep Learning (DL), that is, artificial neural networks~\cite{gong_neural_2006}. DL has a well-proven ability to automatically and efficiently capture complex structures in large volumes of data of all kinds, from simple tabular data to images, videos or sequential data, such as audio or text. Furthermore, neural networks, being parametric methods, provide an explicit function that can be directly and very efficiently applied to out-of-sample points, reduce the need to calculate and store the spectral decomposition of affinity matrices, and they can be easily adapted to different types of data depending on their architecture.

In this work, we propose to alleviate the aforementioned problems with DM by using DL. To this end, we offer a new formulation of the DM embedding as a solution to a certain unconstrained minimization problem. This formulation allows to define a cost function with which to train a neural network to compute the DM embedding both for points inside and outside the training sample, without requiring any spectral decomposition. The capabilities of this approach are compared on different data sets, both real and synthetic, against those of Diffusion Maps and the Nystr\"om method.


The paper is structured as follows. In Section~\ref{sec:spectral_methods}, a theoretical framework is presented on kernel-based spectral methods of manifold learning, Laplacian Eigenmaps, DM and the Nystr\"om method applied to DM. Section~\ref{sec:related_work} explains how spectral methods can be implemented in a general way using DL, and it includes a brief review of the existing works in this regard. Next, in Section~\ref{sec:ddm}, the proposal of this work is presented, which consists of a new formulation of the embedding of DM as a solution to a certain unconstrained minimization problem and the definition of a cost function from it with which to train neural networks. Subsequently, in Section~\ref{sec:experiments}, a series of experiments are proposed to evaluate the presented method, comparing it with DM and the Nystr\"om method, and the obtained results are exposed. Finally, in Section~\ref{sec:conclusions}, the conclusions of the work are collected.

\section{Kernel-based spectral methods}\label{sec:spectral_methods}
Within the set of nonlinear dimensionality reduction methods, there is a subset of them, which will be referred to as kernel-based spectral methods, whose embeddings are calculated from the spectral decomposition of a kernel matrix or some transformation of it. A kernel matrix, $\vb{K}$, is a way to capture the local geometry of a data set, $\Bqty{\vb{x}_i}_{i=1}^N$, from the local similarities between points, measured by what is called a kernel function, $\mathcal{K}:\R^D\times\R^D\to\R$. The elements of the kernel matrix will then be $K_{i,j} = \mathcal{K}(\vb{x}_{i}, \vb{x}_{j})$, for $i, j = 1, \dotsc, N$. This kernel function is typically a function of the Euclidean distance between points in the feature space, such that closer points are considered more similar.
The choice of the kernel function should be guided by the specific application in mind. Typically, the kernel function decays rapidly for increasing values of the distance between points, so the similarity between distant points is close to zero. This is a significant difference from other ``global'' methods, such as PCA~\cite{pearson_liii_1901}, where all correlations between data points are taken into account. Kernel-based spectral methods, on the other hand, start from the idea that, in many applications, the separation between very distant points is irrelevant, so it does not need to be preserved; only the relationships between nearby points constitute the meaningful information in the data set. These methods have the advantages of being nonlinear, capable of generating embeddings that better reflect the local geometry of the data and, due to their local nature, potentially more robust to outliers.

\subsection{Laplacian Eigenmaps}
Laplacian Eigenmaps (LE;~\cite{dietterich_laplacian_2002}) is part of a subset of kernel-based spectral methods that interpret the kernel matrix, $\vb{K}$, as the weight matrix, $\vb{W}$, of the edges of a graph whose nodes are the different data points. Thus, more similar data points are more strongly connected. The embedding vectors, $\Bqty{\vb*{\gamma}_i}_{i=1}^N$, of LE are the solution to the following minimization problem:
\begin{equation}\label{eq:le_min_problem_1}
	\argmin_{\vb*{\gamma}_1,\ldots, \vb*{\gamma}_N\in\R^d}\sum_{i=1}^N\sum_{j=1}^N W_{i,j}\norm{\vb*{\gamma}_i-\vb*{\gamma}_j}^2
\end{equation}
such that
\begin{equation}\label{eq:le_min_problem_restrictions}
    \vb*{\Gamma} \vb{D}_{\vb{W}} \vb*{\Gamma}^\top = \I_d\qq{and} \vb*{\Gamma} \vb{D}_{\vb{W}}\vb{1}_N = \vb{0}_d,
\end{equation}
where $\vb*{\Gamma}\in\R^{d\times N}$ is the matrix whose columns are the vectors $\Bqty{\vb*{\gamma}_i}_{i=1}^N$, $\vb{W}=\vb{K}$ is the weight matrix of the data graph, $\vb{D}_{\vb{W}}\equiv D(\vb{W}) = \operatorname{diag}(d_1(\vb{W}),\dotsc, d_N(\vb{W}))$ is the so-called degree matrix, with $d_i(\vb{W}) = \sum_j W_{i, j}$ the degree of the $i$-th node; and $\vb{1}_N$ and $\vb{0}_d$ are the vectors of ones and zeros of dimensions $N$ and $d$, respectively. LE therefore look for an embedding where the most strongly connected, i.e., most similar, points in the graph are closest. The first constraint prevents solutions with rank less than $d$, including the trivial $\vb*{\Gamma} = 0$, and fixes their normalization,
while the second constraint prevents the other possible trivial solution, in which one of the coordinates has the same value for all points.

The name of LE is due to the fact that
\begin{equation}
	\sum_{i=1}^N\sum_{j=1}^N W_{i,j}\norm{\vb*{\gamma}_i-\vb*{\gamma}_j}^2 = \Trace\pqty{\vb*{\Gamma} \vb{L} \vb*{\Gamma}^\top},
\end{equation}
where $\vb{L}=\vb{D}_{\vb{W}}-\vb{W}$ is the so-called Laplacian of the graph, which is a positive semidefinite quadratic form if $W_{i,j}> 0,\;\forall i,j$. Therefore, the minimization problem \eqref{eq:le_min_problem_1} can be expressed as
\begin{equation}\label{eq:le_min_problem_2}
	\argmin_{\vb*{\Gamma}\in\R^{d\times N}}\Trace\pqty{\vb*{\Gamma} \vb{L} \vb*{\Gamma}^\top} ,
\end{equation}
together with the constraints \eqref{eq:le_min_problem_restrictions}. It can be shown~\cite{vidal_nonlinear_2016} that the solution to this problem is the matrix $\vb*{\Gamma}$ whose rows are the eigenvectors of $\vb{D}^{-1}_{\vb{W}}\vb{L}$ corresponding to the smallest eigenvalues. Since $\vb{1}_N$ is the eigenvector with the smallest eigenvalue of $\vb{D}^{-1}_{\vb{W}}\vb{L}$, that is, $0$, the $d$ eigenvectors must be taken, in the order given by their eigenvalues, starting from the second one, to fulfill the constraints.

\subsection{Diffusion Maps}\label{sec:diffusion_maps}

Another popular nonlinear dimensionality reduction technique is Diffusion Maps (DM;~\cite{coifman_diffusion_2006}). Like LE, DM belongs to the subset of kernel-based spectral methods that interpret the kernel matrix as the weight matrix of a graph connecting the data points. In particular, DM is characterized by defining a random walk on said graph. This, in turn, allows defining a distance between data points, called the diffusion distance, capable of capturing the local geometry\footnote{As in~\cite{coifman_diffusion_2006}, here the term geometry will refer to a set of rules that describe the relationships between data points. For example, ``being close to a point'' is one such rule.} of the data set. The embedding obtained with the so-called diffusion maps allows the Euclidean distances between the data points in the new representation to be related to the diffusion distance between them in the original representation. Thus, it is possible to define a new global representation of the data in a reduced dimension, preserving the most relevant characteristics of its local geometric structure.

In the case of DM, the starting points is a typically Gaussian kernel:
\begin{equation}\label{eq:gaussian_kernel}
	K_{i, j} = \mathcal{K}\pqty{\vb{x}_{i}, \vb{x}_{j}} = \exp\pqty{-\dfrac{\norm{\vb{x}_{i} - \vb{x}_{j}}^2}{2\sigma^2}}.
\end{equation}
In~\cite{coifman_diffusion_2006}, it is proposed to modify this kernel in the following way in order to regulate the influence that the data sampling density has on the algorithm:
\begin{equation}
	K^{(\alpha)}_{i, j} = \dfrac{K_{i, j}}{d_i^\alpha(\vb{K})d_j^\alpha(\vb{K})},
\end{equation}
where $\alpha\in\R$ and $d_i(\vb{K}) = \sum_{j=1}^N K_{i, j}$. 
Defining $\vb{D}_{\vb{K}} = \operatorname{diag}(d_1(\vb{K}),\dotsc, d_N(\vb{K}))$,
we can write
\begin{equation}
	\vb{K}^{(\alpha)} = \vb{D}_{\vb{K}}^{-\alpha} \vb{K} \vb{D}_{\vb{K}}^{-\alpha}.
\end{equation}
The parameter $\alpha$ regulates the influence of the sampling density on the algorithm: when $\alpha = 0$, the original kernel is recovered and the influence is maximum, while for $\alpha=1$ it is null, that is, the statistics and the underlying geometry of the data are decoupled. Once the kernel is set, the weight matrix of the graph is chosen as
\begin{equation}
	\vb{W} = \vb{K}^{(\alpha)}
\end{equation}
and a random walk in the graph can be defined by the following stochastic or Markov matrix:
\begin{equation}\label{eq:dm_prob}
	\vb{P} = \vb{D}_{\vb{W}}^{-1}\vb{W}.
\end{equation}
Specifically, the element $P_{i,j}$ defines the transition probability from node $i$ to node $j$, which does not coincide, in general, with $P_{j, i}$, the transition probability from node $j$ to node $i$. To know the transition probability from node $i$ to node $j$ in $t$ steps, we simply iterate $t$ times over the Markov chain, so that the transition probabilities in $t$ steps between nodes are given by the matrix $\vb{P}^t$. This Markov chain has a stationary distribution, $\vb*{\pi}$, that is, such that $\vb*{\pi}^\top \vb{P}=\vb*{\pi}^\top$, given by
\begin{equation}\label{eq:stationary_dist}
	\pi_i = \dfrac{d_i(\vb{W})}{\sum_j d_j(\vb{W})},
\end{equation}
which is unique if the graph is connected.
The justification for defining this random walk is based on the idea that the kernel, $\vb{K}$, captures relevant aspects of the local geometry of the data, while the Markov chain given by $\vb{P}$ defines, based on the values of $\vb{K}$, a random walk that allows, as it progresses, to propagate and accumulate this information about the local geometry, revealing relevant geometric structures in the data at different scales, and thus globally characterizing the graph. For example, from the point of view of a random walker, a cluster can be understood as a region in the data graph for which the probability of escaping is small.


The transition probabilities given by the different powers of $\vb{P}$ allow to define a family of distances in the graph, called diffusion distances, which reflect the geometry of the data at different scales or resolutions:
\begin{equation}\label{eq:d_dif_orig}
	d^{(t)}_{\text{dif}}\pqty{\vb{x}_{i},\vb{x}_{j}} = \Bqty{\sum_{k=1}^N \dfrac{1}{\pi_k}\bqty{\pqty{\vb{P}^t}_{i, k} - \pqty{\vb{P}^t}_{j, k}}^2}^{1/2}.
\end{equation}
The intuition behind this definition is that the diffusion distance at $t$ steps between two data points, $\vb{x}_{i}$ and $\vb{x}_{j}$, is small if it is approximately equally likely to reach most points in $t$ steps from $\vb{x}_{i}$ as from $\vb{x}_{j}$.
Thus, the notion of proximity that defines this distance reflects the connectivity between nodes in the graph, summarizing all the evidence that relates some data with others. For this reason, the diffusion distance seems especially suitable for detecting structures such as clusters. In general, as already explained, since it is a distance defined from a random walk in the graph, it is capable of capturing very well complex structures in the data at different scales, determined by the value of the parameter $t$. Because the diffusion distance is symmetric and satisfies the triangular inequality, it is a semimetric~\cite{coifman_diffusion_2006}. If, in addition, all the elements of $\vb{K}$ are greater than zero, as it is the case with a Gaussian kernel, then it is a metric. The advantages of diffusion distance over other possible metrics in the graph are its robustness, considering all possible paths of length $t$ between points, and its low computational cost compared to, for example, geodesic distance.

The most interesting aspect of the diffusion distance is that it can be expressed as the Euclidean distance in a embedding of the data defined from the right eigenvectors, $\Bqty{\vb*{\psi}_l}_l$, and eigenvalues, $\Bqty{\lambda_l}_l$, of $\vb{P}$.
To demonstrate this, consider the following matrix:
\begin{equation}\label{eq:matrix_A}
	\vb{A} = \vb{D}_{\vb{W}}^{-1/2} \vb{W} \vb{D}_{\vb{W}}^{-1/2} = \vb{D}_{\vb{W}}^{1/2} \vb{P} \vb{D}_{\vb{W}}^{-1/2} = \vb*{\Pi}^{1/2} \vb{P} \vb*{\Pi}^{-1/2},
\end{equation}
where $\vb*{\Pi} = \operatorname{diag}(\pi_1, \ldots, \pi_N) = \vb{D}_{\vb{W}}/\sum _{i=1}^N d_i(\vb{W})$. Being real and symmetric, there exists a real orthonormal basis, $\Bqty{\vb*{\phi}_l}_{l=1}^N$, of eigenvectors of $\vb{A}$ with real eigenvalues, $\Bqty{\lambda_l}_{l=1}^N$. In turn, since $\vb{P}$ is related to $\vb{A}$ by the change of basis given by $\vb*{\Pi}^{-1/2}$, the eigenvalues of $\vb{P}$ are the same as those of $\vb{A}$, while its left eigenvectors, $\Bqty{\vb*{\varphi}_l}_{l=1}^N$, and right eignevectors, $\Bqty{\vb*{\psi}_l}_{l=1}^N$, are given by
\begin{equation}
	\vb*{\varphi}_l^\top = \vb*{\phi}_l^\top \vb*{\Pi}^{1/2}\qq{and} \vb*{\psi}_l = \vb*{\Pi}^{-1/2}\vb*{\phi}_l.
\end{equation}
Hence, $\vb{P}^t$ can be expressed as $\sum_{l=1}^N\lambda_i^t\vb*{\psi}_l\vb*{\varphi}_l^\top$, that is,
\begin{equation}
	\pqty{\vb{P}^t}_{i, j} = \sum_{l=1}^N\lambda_l^t\pqty{\vb*{\psi}_l}_i\pqty{\vb*{\varphi}_l}_j.
\end{equation}
Substituting this into \eqref{eq:d_dif_orig} we can write
\begin{align}	      \nonumber\bqty{d_{\text{dif}}^{(t)}\pqty{\vb{x}_{i},\vb{x}_{j}}}^2 &= \sum_{k=1}^N \dfrac{1}{\pi_k}\bqty{\sum_{l=1}^N\lambda_l^t\pqty{\vb*{\psi}_l}_i\pqty{\vb*{\varphi}_l}_k - \lambda_l^t\pqty{\vb*{\psi}_l}_j\pqty{\vb*{\varphi}_l}_k}^2\\
	\nonumber&= \sum_{l=1}^N\lambda_l^{2t}\bqty{\pqty{\vb*{\psi}_l}_i - \pqty{\vb*{\psi}_l}_j}^2\sum_{k=1}^N \dfrac{\pqty{\vb*{\varphi}_l}_k^2}{\pi_k}\\
	\nonumber&= \sum_{l=1}^N\lambda_l^{2t}\bqty{\pqty{\vb*{\psi}_l}_i - \pqty{\vb*{\psi}_l}_j}^2\sum_{k=1}^N \dfrac{\bqty{(\vb*{\phi}_l)_k\sqrt{\pi_k}}^2}{\pi_k}\\
	&= \sum_{l = 2}^N \lambda_l^{2 t}\bqty{\pqty{\vb*{\psi}_l}_i-\pqty{\vb*{\psi}_l}_j}^2,\label{eq:d_dif_eig}
\end{align}
where it is taken into account that the eigenvectors $\Bqty{\vb*{\phi}_l}_{l=1}^N$ have norm equal to $1$. Note that the summation of the last line in \eqref{eq:d_dif_eig} starts at $l=2$. The reason is that, if the graph is connected, $\vb{P}$ has a single eigenvalue, $\lambda_1$, equal to $1$, whose eigenvector on the left is $\vb*{\varphi}_1 = \vb*{\pi}$ and on the right, $\vb*{\psi}_1 = \vb{1}_N$; while the rest of its eigenvalues are strictly less than $1$ in absolute value.


From \eqref{eq:d_dif_eig}, it is easy to see that the diffusion distance can be expressed as the Euclidean distance at a certain embedding:
\begin{equation}\label{eq:d_dif_emb}
	d_{\text{dif}}^{(t)}\pqty{\vb{x}_{i},\vb{x}_{j}} = \norm{\Psi^{(t)}\pqty{\vb{x}_{i}}-\Psi^{(t)}\pqty{\vb{x}_{j}}},
\end{equation}
where
\begin{equation}\label{eq:dm_emb_psi}
	\Psi^{(t)}\pqty{\vb{x}_{k}} =\pmqty{
		\lambda_2^t (\vb*{\psi}_1)_k \\
		\vdots \\
		\lambda_N^t (\vb*{\psi}_N)_k
	}
\end{equation}
is the embedding of the $k$-th point.


It is important to note that, to arrive at \eqref{eq:d_dif_eig}, it is necessary to assume that the vectors $\Bqty{\vb*{\phi}_l}_{l=1}^N$ are normalized, so the norm of $\vb*{\psi}_l$ in \eqref{eq:d_dif_eig} and \eqref{eq:dm_emb_psi} is
\begin{equation}
	\norm{\vb*{\psi}_l}^2 = (\vb*{\Pi}^{-1/2}\vb*{\phi}_l)^\top (\vb*{\Pi}^{-1/2}\vb*{\phi}_l) = \vb*{\phi}_l^\top\vb*{\Pi}^{-1}\vb*{\phi}_l.
\end{equation}
Therefore, the choice of the norm of the vectors $\Bqty{\vb*{\psi}_l}_{l=1}^N$ is not arbitrary. In particular, its norm is not, in general, equal to $1$. Consequently, the embedding cannot be calculated directly from the eigenvectors of $\vb{P}$, but it is necessary to do so from the eigenvectors of $\vb{A}$, whose norm is known and equal to $1$:
\begin{equation}\label{eq:dm_emb_phi}
	\Psi^{(t)}\pqty{\vb{x}_{k}} =\dfrac{1}{\sqrt{\pi_k}}\pmqty{
		\lambda_2^t (\vb*{\phi}_2)_k\\
		\vdots \\
		\lambda_N^t (\vb*{\phi}_N)_k
	}.
\end{equation}
The $\Bqty{\Psi^{(t)}}_{t\in\N}$ transformations, which provide a family of embeddings, are called diffusion maps, and each of the components of $\Psi^{(t)}$ is a diffusion coordinate. This family of representations captures the geometry of the data at different scales, depending on the value of $t$. Finally, since all eigenvalues in the summation of \eqref{eq:d_dif_eig} are less than $1$ in absolute value, by taking only the largest magnitude eigenvalues and their corresponding eigenvectors to define \eqref{eq:dm_emb_phi}, and restricting the summation of \eqref{eq:d_dif_eig} to them, one can obtain an approximation of the diffusion distance by the Euclidean distance in a lower-dimensional embedding.
\subsection{Nystr\"om method for Diffusion Maps}\label{sec:dm_nystrom}

Kernel-based spectral DR methods, including DM, present two important problems associated with the spectral decomposition of their similarity matrix. First, its complexity and memory cost scale, respectively, as $\order{N^3}$ and $\order{N^2}$, that is, with the cube and square of the number of data points, $N$. The second major problem is that these methods cannot be directly applied to points not included in the initial data sample. Hence, embedding new data points requires, in principle, re-running the entire algorithm with the extended sample. To try to solve or alleviate these problems, different strategies have emerged. One of the most popular is the so-called Nystr\"om method. This method was initially introduced as a way to obtain numerical solutions to integral equations~\cite{nystrom_uber_1930}. It was later presented as a way to speed up kernel-based algorithms~\cite{williams_using_2000} and, later, as a tool to extend embeddings of spectral methods to points outside the initial sample~\cite{bengio_out--sample_2003}. According to~\cite{williams_using_2000}, Nystr\"om method reduces the memory cost of computing $M<N$ eigenvectors and eigenvalues of $\vb{K}$ from $\order{N^2}$ to $\order{NM}$ and the computational complexity from $\order{N^3}$ to $\order{NM^2}$. The Nystr\"om method for approximating eigenvectors and eigenvalues of a kernel matrix is presented below, as well as how to apply it to the case of DM according to~\cite{fernandez_diffusion_2015}.

Consider two i.i.d. samples, $\Bqty{\vb{x}_{k}}_{k=1}^{M}$ and $\Bqty{\vb{x}_{k}}_{k=M+1}^{N+M}$, from the same distribution and a certain kernel function, $\mathcal{K}$. Let $\lambda^{\aqty{M}}_k$ and $\vb{u}^{\aqty{M}}_k$, respectively, be the $k$-th eigenvalue and its corresponding eigenvector of the kernel matrix, $\vb{K}^{\aqty{M}}$, of the first sample. On the other hand, let $\lambda^{\aqty{N+M}}_k$ and $\vb{u}^{\aqty{N+M}}_k$ denote, respectively, the $k$-th eigenvalue and its corresponding eigenvector of the kernel matrix, $\vb{K}^{\aqty{N+M}}$, of the union of the two samples, $\Bqty{\vb{x}_{k}}_{k=1}^{N+M}$. Then, when $N/M\to 0$, the Nystr\"om method allows to approximate the first $M$ eigenvectors and eigenvalues of the kernel matrix of the union of the two samples from those of the kernel matrix of the first sample as follows~\cite{bengio_out--sample_2003} (see Appendix~\ref{sec:general_nystrom}):
\begin{gather}
	\pqty{\vb{u}^{\aqty{N+M}}_i}_j \approx \dfrac{1}{\lambda_i^{\aqty{M}}}\sum_{k=1}^M\mathcal{K}\pqty{\vb{x}_{j}, \vb{x}_{k}}\pqty{\vb{u}^{\aqty{M}}_i}_k,\label{eq:nystrom_inf_eigenvectors}\\
	\lambda_i^{\aqty{N+M}}\approx\lambda_i^{\aqty{M}},\label{eq:nystrom_inf_eigenvalues}
\end{gather}
for $i = 1,\ldots, M$ and $j = 1, \ldots, N+M$.

In the case of DM, the approximations given by \eqref{eq:nystrom_inf_eigenvectors} and \eqref{eq:nystrom_inf_eigenvalues} can be applied to compute the eigenvectors and eigenvalues of $\vb{A}$, since, like a kernel matrix, it is real and symmetric. The idea is to extend the eigenvectors and eigenvalues of $\vb{A}$ obtained from an initial set, $\Bqty{\vb{x}_{k}}_{k=1}^{M}$, to a different sample, $\Bqty{\vb{x}_{k}}_{k=M+1}^{N+M}$, with \eqref{eq:nystrom_inf_eigenvectors} and \eqref{eq:nystrom_inf_eigenvalues}. However, there is the problem that, unlike the kernel function, $\mathcal{K}$, and the corresponding kernel matrix, $\vb{K}$, there is no function $\mathcal{A}$ such that $A_{i, j} = \mathcal{A}(\vb{x}_i, \vb{x}_j)$ for points both inside and outside the initial sample. On the other hand, recalculating the complete matrix $\vb{A}$ for the extended sample $\Bqty{\vb{x}_{k}}_{k=1}^{N+M}$ each time one wants to obtain the DM embedding on new points is not efficient either in terms of memory or computational complexity. For these reasons, it is necessary to define an approximate function, $\Tilde{\mathcal{A}}$, from the initial sample, $\Bqty{\vb{x}_{k}}_{k=1}^{M}$.

In analogy to the application of Nyström method done in~\cite{bengio_out--sample_2003} for LE and Spectral Clustering, one can define a function that approximates the value of the matrix $\vb{A}$ for new data points, $\vb{a}$ and $\vb{b}$, as
\begin{equation}\label{eq:tilde_A}
	\Tilde{\mathcal{A}}(\vb{a}, \vb{b}) = \dfrac{1}{M}\dfrac{\mathcal{K}^{(\alpha)}(\vb{a},\vb{b})}{\sqrt{\E_{\vb{x}}\bqty{\mathcal{K}^{(\alpha)}(\vb{a},\vb{x})}\E_{\vb{x}}\bqty{\mathcal{K}^{(\alpha)}(\vb{b},\vb{x})}}},
\end{equation}
where the expected values are estimated using the empirical mean over the initial sample. Similarly, when $\alpha\neq 0$, in order to calculate \eqref{eq:tilde_A}, it is also necessary to define
\begin{equation}
	\tilde{\mathcal{K}}^{(\alpha)}(\vb{a},\vb{b}) = \dfrac{1}{M^{2\alpha}}\dfrac{\mathcal{K}(\vb{a},\vb{b})}{\E_{\vb{x}}\bqty{\mathcal{K}(\vb{a},\vb{x})}^\alpha\E_{\vb{x}}\bqty{\mathcal{K}(\vb{b},\vb{x})}^\alpha}
\end{equation}
to approximate $\mathcal{K}^{(\alpha)}$, where $\mathcal{K}$ is the original kernel function. If $\alpha=0$, this is not necessary, since then $\mathcal{K}^{(\alpha)} = \mathcal{K}$, whose analytical expression is known. Note that $\tilde{\mathcal{A}}(\vb{x}_i, \vb{x}_j) = A_{i, j}^{\aqty{M}}$ for $i, j = 1, \ldots, M$, since
\begin{gather}
    d_i\pqty{\vb{K}^{\aqty{M}}} = M\E_{\vb{x}}\bqty{\mathcal{K}\pqty{\vb{x}_i, \vb{x}}},\\
    d_i\pqty{\vb{W}^{\aqty{M}}} = M\E_{\vb{x}}\bqty{\tilde{\mathcal{K}}^{(\alpha)}\pqty{\vb{x}_i, \vb{x}}}.
\end{gather}
when the expected values are estimated using the empirical mean over the initial sample. With these approximations, to extend the eigenvectors of $\vb{A}$, it would only be necessary to keep $\Bqty{d_i\pqty{\vb{K}^{\aqty{M}}}}_{i=1}^M$ and $\Bqty{d_i\pqty{\vb{W}^{\aqty{M}}}}_{i=1}^M$, and compute $\mathcal{K}(\vb{x}_j, \vb{x}_i)$ for $i = 1, \ldots, M$ and $j = M+1, \ldots, N+M$.

Finally, once the extensions of the eigenvectors of $\vb{A}$ have been calculated, in order to calculate the embedding \eqref{eq:dm_emb_phi} from them, it is necessary to know the value of the stationary distribution, $\vb*{\pi}$, not only at the points of the reduced sample, but at all the points of the complete sample. To achieve this, we can extend the eigenvector $\vb*{\phi}_1$ with eigenvalue $\lambda_1 = 1$, since this eigenvector coincides with the square root of the stationary distribution over the corresponding points: $\vb*{\phi}_1 = \sqrt{\vb*{\pi}}\equiv (\sqrt{\pi_1},\ldots, \sqrt{\pi_N})^\top$.
\section{Related works}\label{sec:related_work}

As already mentioned, a more recent alternative to the Nystr\"om method consists in turning to deep learning. The motivation is threefold. First, neural networks are parametric methods, meaning that they provide an explicit function that can be evaluated both within and outside the initial data set. Second, neural networks have a proven ability to generate complex and efficient representations of large volumes of data of all kinds. Finally, neural network training and inference can now be very fast thanks to the use of specialized hardware such as GPUs (Graphics Processing Units) and TPUs (Tensor Processing Units).

The idea of using neural networks to map each data point to the corresponding embedding of some spectral manifold learning method was first studied in~\cite{gong_neural_2006}. Although there are other works that study ways to improve the efficiency of spectral methods through deep learning, such as~\cite{tian_learning_2014, han_mini-batch_2017}, only the approach presented in~\cite{gong_neural_2006} allows training a neural network that receives as input an instance of the original feature space and associates its corresponding embedding. This aspect is key, since it is what allows direct application to data outside the training sample. As explained in~\cite{gong_neural_2006}, given a data set, $\mathcal{D} = \Bqty{\vb{x}_i}_{i=1}^N$, spectral manifold learning methods typically define an embedding, $\Bqty{\vb*{\gamma}_i}_{i=1}^N$, that solves some minimization problem of the form
\begin{equation}
	\argmin_{\vb*{\gamma}_1,\ldots, \vb*{\gamma}_N\in\R^d}\mathcal{Q}\pqty{\Bqty{\vb{x}_i}_{i=1}^N, \Bqty{\vb*{\gamma}_i}_{i=1}^N}.
\end{equation}
According to~\cite{gong_neural_2006}, there are two possible ways to solve the same problem using deep learning:
\begin{enumerate}
	\item Compute the embedding vectors, $\Bqty{\vb*{\gamma}_i}_{i=1}^N$, for the training data set, $\mathcal{D}$, using the original spectral method and train a neural network, $f_{\vb*{\theta}}$ with parameters $\vb*{\theta}$, to solve the regression problem associated with the labeled data set $\Bqty{\pqty{\vb{x}_i, \vb*{\gamma}_i}}_{i=1}^N$. One possible way to do this is by using the following cost function:
    \begin{equation}\label{eq:metodo_neuronal_1}
		\mathcal{J}(\vb*{\theta}) = \dfrac{1}{N}\sum_{i=1}^N\norm{f_{\vb*{\theta}}\pqty{\vb{x}_i} - \vb*{\gamma}_i}^2.
	\end{equation}
	\item Train a neural network, $f_{\vb*{\theta}}$, taking as cost function the result of replacing the embedding vectors with the network outputs in the objective function:
    \begin{equation}\label{eq:metodo_neuronal_2}
		\mathcal{J}(\vb*{\theta}) = \mathcal{Q}\pqty{\Bqty{\vb{x}_i}_{i=1}^N, \Bqty{f_{\vb*{\theta}}\pqty{\vb{x}_i}}_{i=1}^N}.
	\end{equation}
\end{enumerate}

As shown in~\cite{jansen_scalable_2017} for the first option, neural networks are able to provide better embedding vectors in less time than the Nystr\"om method for points outside the training sample. 
In fact, the runtime gain can reach several orders of magnitude for large data sets. Furthermore, neural networks, once trained, do not need to continue storing the training data or the results of the spectral decomposition, as happens in Nystr\"om, so the memory savings can also be significant.

However, the first approach, although effective, does not avoid having to perform the expensive spectral decomposition of a kernel matrix in order to carry out the training. The second approach, instead, manages to avoid this; 
indeed, it has been successfully applied to the case of Isomap~\cite{tenenbaum_global_2000} in~\cite{pai_deep_2022} and also to the case of Spectral Clustering~\cite{jianbo_shi_normalized_2000} in~\cite{shaham_spectralnet_2018}. Furthermore, this second method allows to employ the versatility and power of deep learning to easily extend classical linear methods such as PCA~\cite{pearson_liii_1901}, with Autoencoders~\cite{kramer_nonlinear_1991}, or FLDA (Fisher Linear Discriminant Analysis)~\cite{fisher_use_1936}, with DFDA (Deep Fisher Discriminant Analysis)~\cite{diaz-vico_deep_2020}, to a nonlinear format.
\section{Deep Diffusion Maps}\label{sec:ddm}

The proposal of this work consists in applying the approach given by \eqref{eq:metodo_neuronal_2} to the case of DM. To do so, it is necessary to relate the embedding of DM with the solution of some minimization problem. One way to do this is by realizing that the embedding of DM is a scaled version of the embedding of LE when in both of them the weights of the data graph are defined by the same kernel matrix.

The reason is that
\begin{equation}
	\vb{D}^{-1}_{\vb{W}}\vb{L} = \I_N - \vb{D}^{-1}_{\vb{W}}\vb{W} = \I_N - \vb{P},
\end{equation}
so the right eigenvectors of $\vb{D}^{-1}_{\vb{W}}\vb{L}$, which define the LE embedding, are the same as those of $\vb{P}$, which define the DM solution. Furthermore, the relation between the eigenvalues of $\vb{D}^{-1}_{\vb{W}}\vb{L}$, $\{\tilde{\lambda}_i\}_{i=1}^N$, and the eigenvalues of $\vb{P}$ corresponding to the same eigenvectors, $\{\lambda_i\}_{i=1}^N$, is
\begin{equation}
	\tilde{\lambda}_i = 1 - \lambda_i, \qquad i = 1,\ldots, N,
\end{equation}
so the eigenvectors associated with the smallest eigenvalues in LE correspond to the eigenvectors with the largest eigenvalues in DM.

However, the DM embedding, although it satisfies the second constraint of the LE minimization problem \eqref{eq:le_min_problem_restrictions}, it does not verify the first one due to its normalization. According to \eqref{eq:dm_emb_phi}, the matrix whose columns are the embedding vectors of DM is $\vb*{\Gamma} = \vb*{\Lambda}^{t}(\vb*{\Pi}^{-1/2}\vb*{\Phi})^\top$, where $\vb*{\Phi}$ is the matrix whose columns are the eigenvectors, $\Bqty{\vb*{\phi}_l}_{l=2}^N$, of $\vb{A}$ and $\vb*{\Lambda} = \operatorname{diag}\pqty{\lambda_2, \ldots, \lambda_N}$, with $\lambda_2\geq\ldots\geq\lambda_N$ the corresponding eigenvalues. Therefore,
\begin{equation}
	\vb*{\Gamma} \vb{D}_{\vb{W}} \vb*{\Gamma}^\top = \pqty{\vb*{\Lambda}^t\vb*{\Phi}^\top\vb*{\Pi}^{-1/2}}\vb{D}_{\vb{W}}\pqty{\vb*{\Pi}^{-1/2}\vb*{\Phi}\vb*{\Lambda}^t}= \vb*{\Lambda}^{2t}\sum_{i=1}^N d_i(\vb{W}).\label{eq:dm_normalization_restriction}
\end{equation}

In conclusion, DM solves the same minimization problem as LE but with a different normalization constraint. Therefore, DM also defines an embedding that tries to minimize the Euclidean distance between the most strongly connected points in the graph, that is, the most similar, but with the advantage that, at the same time, it approximates the diffusion distance using the Euclidean distance.

However, this is a constrained minimization problem, which is more difficult to solve using neural networks. Furthermore, the constraint \eqref{eq:dm_normalization_restriction} necessarily involves performing a spectral decomposition of $\vb{A}$ in order to define a cost function, something that the approach given by \eqref{eq:metodo_neuronal_2} attempted to avoid. Given these problems, a particularly convenient way to define the embedding of DM is the one proposed for the first time in this work through the following theorem.

\begin{thm}
The matrix whose columns are the embedding vectors, $\Bqty{\vb*{\gamma}_i}_{i=1}^N$, of dimension $d$, given by Diffusion Maps is $\vb{R}\vb*{\Gamma}^\star$, with $\vb{R}\in O(d)$ a certain orthogonal matrix\footnote{That is, a rotation and/or reflection.} and
\begin{equation}\label{eq:dm_min_problem}
		\vb*{\Gamma}^\star = \argmin_{\vb*{\Gamma}\in\R^{d\times N}}\norm{\vb*{\Pi}^{1/2}\vb*{\Gamma}^\top \vb*{\Gamma}\vb*{\Pi}^{1/2} - \pqty{\vb{A}^{2t}-\sqrt{\vb*{\pi}}\sqrt{\vb*{\pi}}^\top}}_F^2,
	\end{equation} 
where $\vb{A}$ is given by \eqref{eq:matrix_A}, $\vb*{\Pi} = \operatorname{diag}(\pi_1, \ldots, \pi_N)$ and $\vb*{\pi}$ is the stationary distribution \eqref{eq:stationary_dist}.
\end{thm}
\begin{proof}
If $\Bqty{\vb*{\phi}_l}_{l=1}^N$ are the eigenvectors of $\vb{A}$ and $\Bqty{\lambda_l}_{l=1}^N$, with $1 = \lambda_1\geq\cdots\geq\lambda_N$, are their corresponding eigenvalues, then $\sqrt{\vb*{\pi}} = \vb*{\phi}_1$ and
\begin{equation}
		\vb{A}^{2t}-\sqrt{\vb*{\pi}}\sqrt{\vb*{\pi}}^\top = \pqty{\sum_{l=1}^N\lambda_l^{2t} \vb*{\phi}_l\vb*{\phi}_l^\top} - \vb*{\phi}_1\vb*{\phi}_1^\top = \sum_{l=2}^N\lambda_l^{2t} \vb*{\phi}_l\vb*{\phi}_l^\top = \vb*{\Phi}\vb*{\Lambda}^{2t}\vb*{\Phi}^\top,
	\end{equation}
where $\vb*{\Phi}$ is the matrix whose columns are the eigenvectors $\Bqty{\vb*{\phi}_l}_{l=2}^N$ and $\vb*{\Lambda} = \operatorname{diag}(\lambda_2,\ldots, \lambda_N)$. Therefore, $\vb*{\Phi}$ and $\vb*{\Lambda}^{2t}$ are, respectively, the matrices of singular vectors and values of $\vb{A}^{2t}-\sqrt{\vb*{\pi}}\sqrt{\vb*{\pi}}^\top$. Hence, if we consider the partition $\vb*{\Phi} = \pmqty{\vb*{\Phi}_1 & \vb*{\Phi}_2}$ and $\vb*{\Lambda} = \operatorname{diag}(\vb*{\Lambda}_1, \vb*{\Lambda}_2)$, where $\vb*{\Lambda}_1 = \operatorname{diag}\pqty{\lambda_2, \ldots, \lambda_d}$ and $\vb*{\Phi}_1$ is the matrix whose columns are the corresponding eigenvectors, then by the Eckart–Young–Mirsky theorem\footnote{Informally, the theorem asserts that the matrix $\vb{C}$, with $\rank\pqty{\vb{C}}\leq d$, which best approximates, according to the Frobenius norm, another certain matrix $\vb{B}$ is the one obtained by truncating the singular value decomposition of $\vb{B}$ to the first $d$ singular values. In this case, the problem constraint would be $\rank\pqty{\vb*{\Pi}^{1/2}\vb*{\Gamma}^\top \vb*{\Gamma}\vb*{\Pi}^{1/2}}\leq d$, but it is automatically fulfilled, since $\rank\pqty{\vb{M}_1\vb{M}_2}\leq\min\Bqty{\rank\pqty{\vb{M}_1}, \rank\pqty{\vb{M}_2}}$ and $\rank\pqty{\vb*{\Gamma}}\leq d$, since $\vb*{\Gamma}$ has dimensions $d\times N$.}~\cite{eckart_approximation_1936, mirsky_symmetric_1960}, it follows that
\begin{equation}
		\vb*{\Pi}^{1/2}\pqty{\vb*{\Gamma}^{\star}}^\top \vb*{\Gamma}^\star\vb*{\Pi}^{1/2} = \vb*{\Phi}_1\vb*{\Lambda}_1^{2t}\vb*{\Phi}_1^\top.
\end{equation}
Consequently,
\begin{equation}
		\vb*{\Gamma}^\star\vb*{\Pi}^{1/2} = \tilde{\vb{R}}\vb*{\Lambda}^t_1\vb*{\Phi}^\top_1
\end{equation}
for any orthogonal matrix $\tilde{\vb{R}}\in O(d)$. From the above result, we finally obtain
\begin{equation}
		\vb*{\Gamma}^\star = \tilde{\vb{R}}\vb*{\Lambda}^t_1\vb*{\Phi}^\top_1\vb*{\Pi}^{-1/2}.
\end{equation}
On the other hand, the embedding \eqref{eq:dm_emb_phi} of DM of dimension $d$ can be expressed in matrix form as $\vb*{\Lambda}_1^t\pqty{\vb*{\Pi}^{-1/2}\vb*{\Phi}_1}^\top = \vb*{\Lambda}_1^t\vb*{\Phi}_1^\top\vb*{\Pi}^{-1/2}$. Hence, $\vb{R}\vb*{\Gamma}^\star$, with $\vb{R}=\tilde{\vb{R}}^\top$, is the matrix whose columns are the embedding vectors of DM.
\end{proof}

The problem \eqref{eq:dm_min_problem} can also be expressed as
\begin{equation}
	\argmin_{\vb*{\gamma}_1,\ldots,\vb*{\gamma}_N\in\R^d}\sum_{i=1}^N\sum_{j=1}^N\Bqty{\sqrt{\pi_i\pi_j}\vb*{\gamma}_i^\top\vb*{\gamma}_j - \bqty{\pqty{\vb{A}^{2t}}_{i,j}-\sqrt{\pi_i\pi_j}}}^2.
\end{equation}
Therefore, a neural network, $f_{\vb*{\theta}}:\R^D\to\R^d$, can be trained to obtain the embedding vector, $\vb*{\gamma}$, of DM for any instance of the original feature space, $\vb{x}$, received as input using the following cost function:
\begin{equation}\label{eq:cost_ddm}
	\mathcal{J}(\vb*{\theta}) = \sum_{i=1}^N\sum_{j=1}^N\Bqty{\sqrt{\pi_i\pi_j}f_{\vb*{\theta}}\pqty{\vb{x}_i}^\top f_{\vb*{\theta}}\pqty{\vb{x}_j} - \bqty{\pqty{\vb{A}^{2t}}_{i,j}-\sqrt{\pi_i\pi_j}}}^2.
\end{equation}
As a result, the neural network will be a parametric function that approximates the diffusion map of $t$ steps for points inside and outside the training sample:
\begin{equation}
	f_{\vb*{\theta}}(\vb{x})\approx\Psi^{(t)}(\vb{x}).
\end{equation}
To refer to this method that has just been presented to obtain the embedding of DM through deep learning, the name Deep Diffusion Maps (DDM) will be used. To the best of the authors' knowledge, this work is the first to use deep learning to obtain the embedding vectors of DM directly from the corresponding vectors in the original space and without the need to perform any spectral decomposition, giving for this purpose a new formulation of DM as a solution to a certain unconstrained minimization problem.

Just as DM can be applied to data vectors, images, or functional data~\cite{barroso_functional_2024}, this approach also allows for this. In the case of data vectors (tabular data), one can simply use dense or feedforward neural networks (FFNNs). With images, one can apply convolutional neural networks (CNNs) or, if the images are treated as pixel vectors, also FFNNs. For functional or sequential data, one can use both one-dimensional CNNs and recurrent neural networks (RNNs) and their variants. Functional data can also be expressed, approximately, as vectors in a certain base of functions, which can be worked with as in the case of tabular data. A key aspect in all cases is the calculation of the Euclidean distance between data points in order to calculate the Gaussian kernel. When it comes to images, this is done using the Frobenius norm, which is the same as using the usual Euclidean norm by treating images as vectors of pixels. For functional data, the generalization of the Euclidean distance is given by the $L_2$ functional norm. When the functional data are vectors of equally spaced values of the corresponding functions, then the $L_2$ norm reduces to the Euclidean norm.
\section{Experiments}\label{sec:experiments}

In this section, the results obtained by applying Deep Diffusion Maps (DDM; Section~\ref{sec:ddm}) on different datasets are compared with those offered by the Nystr\"om method (Section~\ref{sec:dm_nystrom}) and the original Diffusion Maps algorithm (DM; Section~\ref{sec:diffusion_maps}). All the code and the results of the experiments can be found in the repository of this work on GitHub\footnote{\url{https://github.com/sgh14/deep-diffusion-maps}}.

\subsection{Methodology}\label{subsec:methods}

The experiments will consist in taking different data sets and dividing each of them into two subsets, $\mathcal{D}_a$ and $\mathcal{D}_b$. First, the original DM algorithm will be applied to the entire set, that is, to $\mathcal{D}\equiv \mathcal{D}_a\cup \mathcal{D}_b$. Second, the subset $\mathcal{D}_a$ will be used as the training set for DDM and, once the model is trained, the embedding vectors for $\mathcal{D}_b$ will be calculated. On the other hand, the DM embedding will be calculated for the subset $\mathcal{D}_b$ from the sample $\mathcal{D}_a$ using the Nystr\"om method. Finally, the embeddings obtained for $\mathcal{D}_a$ and $\mathcal{D}_b$ using DDM and the Nystr\"om method will be compared with respect to the original DM algorithm.
The data sets used, the neural network architectures employed, the choice of hyperparameters and the metrics selected to evaluate the results are detailed below.

\subsubsection{Data sets}

In order to compare the proposed method on data of different types and with underlying manifolds with different degrees of definition and complexity, five different data sets have been considered: Swiss Roll, S Curve, Helix, MNIST and Phoneme.

First, Swiss Roll\footnote{\url{https://scikit-learn.org/stable/modules/generated/sklearn.datasets.make_swiss_roll.html}} is a synthetic dataset, consisting of a random sample of points drawn from a rolled two-dimensional manifold, defined by
    $\vb{x} = \pqty{t\cos(t), h, t\sin(t)}^\top$,
with $t\in[3\pi/2, 9\pi/2)$ and $ h\in[0, 21)$. For this work, a sample of \num{2000} points will be used, generated from uniform distribution samples of $t$ and $h$. Of these \num{2000} points, half will form the subset $\mathcal{D}_a$ and the other half, the subset $\mathcal{D}_b$.

Second, S Curve\footnote{\url{https://scikit-learn.org/stable/modules/generated/sklearn.datasets.make_s_curve.html}} is another synthetic dataset, again a random sample of points drawn from a two-dimensional manifold, this time S-shaped and defined by
    $\vb{x} = \pqty{\sin(t), h, \operatorname{sign}(t)[\cos(t)-1]}^\top$,
with $t\in[-3\pi/2, 3\pi/2)$ and $h\in[0, 2)$. As before, a sample of \num{2000} points will be used, generated from uniform distribution samples of $t$ and $h$. Of these \num{2000} points, half will form the subset $\mathcal{D}_a$ and the other half, $\mathcal{D}_b$.

Thirdly, Helix, as in the two previous cases, is an artificial data set, in this case formed by a sample of points extracted from a one-dimensional manifold defined by
    $\vb{x} = \pqty{\cos(\theta), \sin(2\theta), \cos(3\theta)}^\top$,
with $\theta\in[0, 2\pi)$. A sample of \num{2000} points will be used, generated from a uniform distribution sample of $\theta$. Half of these \num{2000} points will form the set $\mathcal{D}_a$ and the other half, the set $\mathcal{D}_b$.

On the other hand, MNIST\footnote{\url{https://www.tensorflow.org/datasets/catalog/mnist?hl=es}} is a set of \num{70000} images of dimensions $28\times 28$ in gray scale of the handwritten digits 0 to 9~\cite{le_cun_handwritten_1989}. In order to facilitate image processing, each pixel is divided by 255, to ensure that its value is in the range $[0, 1]$. For computational cost reasons, the set is reduced to a random sample of \num{5000} instances of the digits 0 to 5. Of those \num{5000} instances, \num{4000} will be used in the subset $\mathcal{D}_a$ and the remaining \num{1000} in $\mathcal{D}_b$. Unlike the previous sets, this is a real data set.

Finally, Phoneme\footnote{\url{https://fda.readthedocs.io/en/stable/modules/autosummary/skfda.datasets.fetch_phoneme.html}} is a set of \num{4509} periodograms of length 256 associated with 5 different phonemes of the English language: ``sh'' (\num{872} instances), as in \textit{she}; ``dcl'' (\num{757} instances), as in \textit{dark}; ``iy'' (\num{1163} instances), as in the vowel in \textit{she}; ``aa'' (\num{695} instances), as in the vowel in \textit{dark}; and ``ao'' (\num{1022} instances), as in the first vowel in \textit{water}. As in~\cite{ferraty_nonparametric_2006}, of the total of 256 frequencies in each periodogram, only the first 150 are considered, since the lowest frequencies provide the most useful information. Of the \num{4509} periodograms, \num{1000} are used for the subset $\mathcal{D}_b$ and the rest for the subset $\mathcal{D}_a$.

In Figure~\ref{fig:data-train}, a representation of the training data for the sets just presented can be seen.

\begin{figure}[htbp]
		\centering
		\captionsetup{justification=centering}
            \begin{subfigure}[t]{0.28\columnwidth}
			\centering
			\includegraphics[width=\textwidth]{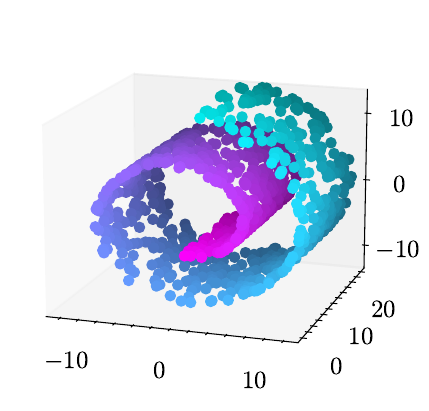}
			\caption{Swiss Roll.}
			\label{perfil-B2}
		\end{subfigure}
        \begin{subfigure}[t]{0.28\columnwidth}
			\centering
			\includegraphics[width=\textwidth]{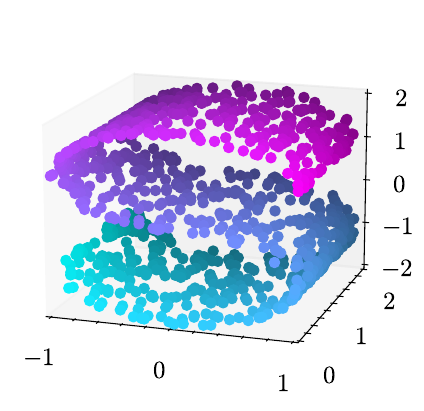}
			\caption{S Curve.}
			\label{perfil-B2}
		\end{subfigure}
		\begin{subfigure}[t]{0.28\columnwidth}
			\centering
			\includegraphics[width=\textwidth]{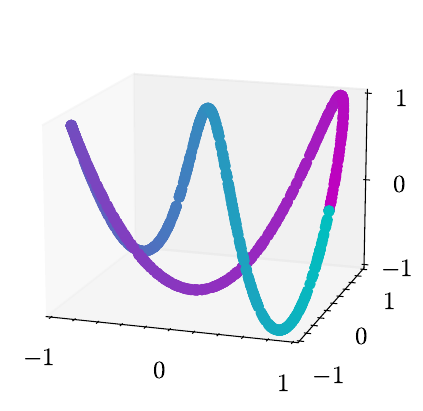}
			\caption{Helix.}
			\label{imagen-B2}
		\end{subfigure}\\
		\begin{subfigure}[t]{0.24\columnwidth}
			\centering
			\includegraphics[width=\textwidth]{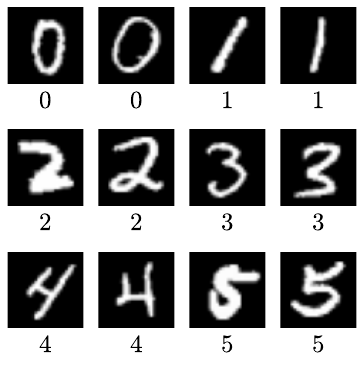}
			\caption{MNIST.}
			\label{perfil-B2}
		\end{subfigure}
		\begin{subfigure}[t]{0.24\columnwidth}
			\centering
			\includegraphics[width=\textwidth]{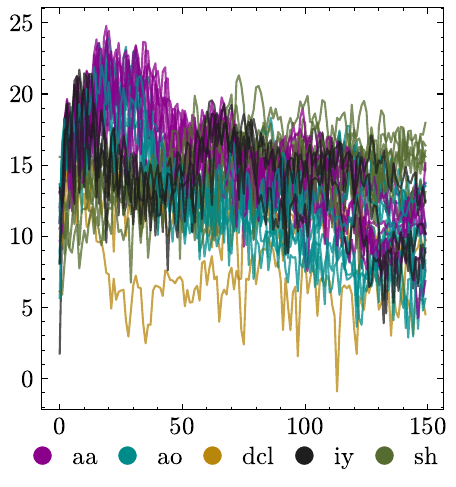}
			\caption{Phoneme.}
			\label{imagen-B3}
		\end{subfigure}
		\caption{Subset $\mathcal{D}_a$ for each data set.}
		\label{fig:data-train}
\end{figure}

\subsubsection{Neural network architectures}

In order to apply DDM to different types of data, three classes of deep neural networks have been defined:
\begin{itemize}
	\item \textbf{Dense network}: It is applied to tabular data, such as Swiss Roll, S Curve, and Helix. The network consists of three dense hidden layers, with 128, 64, and 32 neurons, respectively, and a ReLU activation function, followed by a dense output layer with linear activation and a number of neurons equal to the embedding dimension.
	\item \textbf{Convolutional network}: It is designed to process images, such as those from MNIST. The network is made up of 3 convolutional blocks, followed by a dense hidden layer and an also dense output layer. Each convolutional block is made up of two convolutional layers~\cite{le_cun_handwritten_1989}, with the same number of $3\times 3$ filters and ReLU activation, followed by a max pooling layer of dimensions $2\times 2$ and a dropout regularization layer~\cite{srivastava_dropout_2014}. The convolutional layers of the first block both have 16 filters; those of the second, 32; and those of the third, 64. On the other hand, the dense hidden layer uses the ReLU function as activation and a number of neurons equal to 16 times the dimension of the embedding. Finally, the dense output layer has linear activation and as many neurons as the number of dimensions of the embedding.
	\item \textbf{Recurrent network}: It is designed to work with sequential data, such as Phoneme. The network is made up of three bidirectional~\cite{schuster_bidirectional_1997} LSTM (Long Short-Term Memory)~\cite{hochreiter_long_1997} layers, with 128, 64 and 32 neurons, respectively, followed by a dense output layer with a number of neurons equal to the dimension of the embedding. The LSTM layers have the hyperbolic tangent as their main activation function and the sigmoid function for the recurrent part, while the dense output layer has a linear activation.
\end{itemize}

\subsubsection{Hyperparameters}\label{sec:hyperparameters}

The hyperparameters in the experiments of this work can be divided into two groups: those of DM and those related to deep learning models. The hyperparameters of DM are: $d$, the dimension of the embedding, and $\sigma$, $\alpha$ and $t$. It is a common practice to choose $\sigma$ as a certain quantile, $q\in[0, 1]$, of the Euclidean distances between data points in the original space. On the other hand, the dimension of the embedding can be chosen following the procedure proposed in~\cite{zhu_automatic_2006}. The idea is to calculate a likelihood curve that measures how different the distribution of the retained eigenvalues (raised to $t$) is with respect to that of the discarded eigenvalues as a function of $d$ and to choose the value of $d$ that maximizes said likelihood, that is, the one that maximizes the difference between both distributions. On the other hand, as far as it has been found, there is no known universally valid and objective method for selecting $q$, $\alpha$ and $t$. However, the objective of this work is not to find optimal values for these parameters ---if it is possible to define when they are optimal---, but to compare the embeddings obtained by the methods considered for reasonable values of these parameters in each data set.

In the case of the synthetic data sets (Swiss Roll, S Curve and Helix), we have taken advantage of the fact that the dimension of the original space is $D=3$ to visualize the underlying manifold of the data and choose the hyperparameters in the following way. A set of possible values for $q$ and another for $t$ have been defined, the largest value of $q$ capable of correctly unfolding the manifold has been taken and, once $q$ has been fixed, the smallest value of $t$ such that the dimension suggested for the embedding by the likelihood curve matches the expected one has been taken\footnote{Swiss Roll and S Curve are clearly manifolds of intrinsic dimension equal to $2$, while Helix is a closed one-dimensional manifold, so two dimensions are also necessary to represent the manifold without cutting it.}: in all three cases, $d=2$. For $q$, values between $0$ and \num{5E-2} have been tested with a step of \num{5E-3}, while for $t$ the values $1$, $25$, $50$ and $100$ have been tried. Due to the strongly geometric nature of Swiss Roll, S Curve and Helix, $\alpha=1$ has been taken, in order to eliminate the influence that the data sampling density has on the calculation of the embedding. The resulting combination of $q$, $t$ and $\alpha$ is also visually validated by representing the distribution of the transition probabilities in $t$ steps \eqref{eq:dm_prob} and the corresponding diffusion distances \eqref{eq:d_dif_orig} on the data in the original space and checking that they correctly capture the underlying manifold.

Regarding the real data sets (MNIST and Phoneme), since the dimension of the original space is greater than $3$ in both cases, it is not possible to visualize the manifold of the data and follow the same procedure as for the synthetic sets. For this reason, combined with the lack of another method to adjust $q$, $t$ and $\alpha$ and taking into account that the objective of the work is not to find optimal values for these parameters, it has been decided to take, in both cases, the smallest value of $q$ used in the synthetic sets and $t=1$. On the other hand, $\alpha = 0$ will be set so as not to eliminate the influence of the data sampling density, which in this case can be informative. As for the dimension of the embedding, the one suggested by the likelihood curve will be taken. The experiments will also be repeated for $d=2$, in order to visualize the relative rotation of the embedding of DDM with respect to that of DM.

Finally, regarding the hyperparameters of deep learning models, beyond those already detailed about their architecture, in all cases Adam~\cite{kingma_adam_2014} has been used as an optimization method, with a learning rate equal to \num{E-2} and a batch size of 512 instances. The number of epochs has been extended in each case until the cost function saturates\footnote{In reality, a much smaller number of epochs could have been used in all cases and very similar results would have been obtained, but it was preferred to ensure that the cost functions saturate completely or almost completely.}. Specifically, \num{5000} training epochs have been used for Swiss Roll, S Curve and Helix, while for MNIST \num{10000} have been required and for Phoneme, \num{2500}. On the other hand, from each training set, $\mathcal{D}_a$, \SI{10}{\percent} of the points have been extracted for the validation set. In the case of the convolutional network for MNIST, there is an additional hyperparameter, the dropout probability, although it has finally been set to $0$.

\subsubsection{Metrics}

To assess the quality of the results in a quantitative way, the Euclidean distances between each pair of points, $\vb*{\gamma}_i$ and $\vb*{\gamma}_j$, in the embedding offered by DDM and the Nystr\"om method will be measured, and the mean relative error,
\begin{equation}
    \mathrm{MRE} = \dfrac{2}{N(N-1)}\sum_{i < j}\dfrac{\abs{\norm{\vb*{\gamma}_i-\vb*{\gamma}_j} - \norm{\tilde{\vb*{\gamma}}_i-\tilde{\vb*{\gamma}}_j}}}{\norm{\tilde{\vb*{\gamma}}_i-\tilde{\vb*{\gamma}}_j}},
\end{equation}
of these distances with respect to the distances between the same points in the embedding of DM, $\tilde{\vb*{\gamma}}_i$ and $\tilde{\vb*{\gamma}}_j$, will be calculated. This quantity will be calculated both on the subset $\mathcal{D}_a$ and on $\mathcal{D}_b$. Note that this error measure is invariant to orthogonal transformations $\vb{R}\in O(d)$ of the embedding, so it allows comparing the results of DDM with those of DM.
\subsection{Results}\label{subsec:results}
In this section, the results obtained in the experiments outlined in the previous section are presented. First, Table~\ref{tab:hyperparameters} contains the selected values of $d$, $q$, $\alpha$ and $t$ for each data set as explained in Section~\ref{sec:hyperparameters}.

\begin{table}[htbp]
	\centering
    \captionsetup{justification=centering}
    \caption{Diffusion Maps hyperparameter values.}
	\label{tab:hyperparameters}
		\begin{tabular}{l S[table-format=1] c S[table-format=1] S[table-format=3]}
			\toprule
			\textbf{Data set} &   $\vb*{d}$ &   $\vb*{q}$ & $\vb*{\alpha}$ & $\vb*{t}$\\
			\midrule
			Swiss Roll & 2 & \num{5E-3} & 1 & 100\\
            S Curve & 2 & \num{5E-3} & 1 & 100\\
			Helix & 2 & \num{3E-2} & 1 & 100\\
			MNIST & 6 &  \num{5E-3} & 0 & 1\\
			Phoneme & 2 & \num{5E-3} & 0 & 1\\
			\bottomrule
		\end{tabular}
\end{table}


In Figure~\ref{fig:probs-a}, the transition probability distributions \eqref{eq:dm_prob} with respect to the point marked in yellow are shown for the three synthetic data sets (Swiss Roll, S Curve and Helix), obtained using the hyperparameters in Table~\ref{tab:hyperparameters}. In Figure~\ref{fig:dists-a}, the same is done for the diffusion distance \eqref{eq:d_dif_orig}, where the reference point is marked, again, in yellow. Note that darker colors denote smaller values in both cases. As it can be seen, the choice of hyperparameters seems appropriate, as it correctly captures the local geometry of the data, that is, the probabilities and distances accurately reflect the proximity between points within the underlying manifold.

\begin{figure}[htbp]
		\centering
		\captionsetup{justification=centering}
            \begin{subfigure}[t]{0.28\columnwidth}
			\centering
			\includegraphics[width=\textwidth]{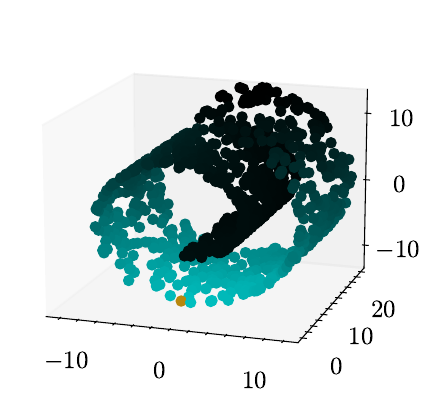}
			\caption{Swiss Roll.}
			\label{perfil-B2}
		\end{subfigure}
        \begin{subfigure}[t]{0.28\columnwidth}
			\centering
			\includegraphics[width=\textwidth]{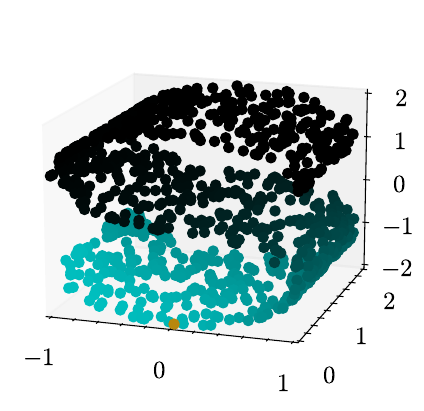}
			\caption{S Curve.}
			\label{perfil-B2}
		\end{subfigure}
		\begin{subfigure}[t]{0.28\columnwidth}
			\centering
			\includegraphics[width=\textwidth]{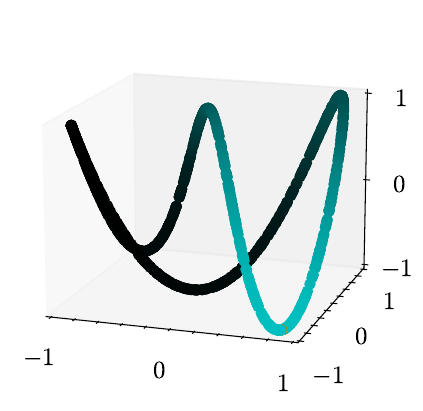}
			\caption{Helix.}
			\label{imagen-B2}
		\end{subfigure}
		\caption{Transition probabilities in $t$ steps with respect to the yellow point for the set $\mathcal{D}_a$.}
		\label{fig:probs-a}
\end{figure}

\begin{figure}[htbp]
		\centering
		\captionsetup{justification=centering}
            \begin{subfigure}[t]{0.28\columnwidth}
			\centering
			\includegraphics[width=\textwidth]{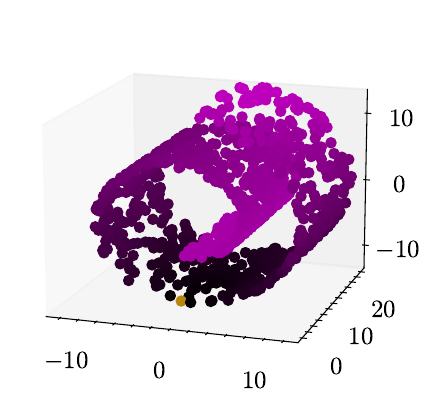}
			\caption{Swiss Roll.}
			\label{perfil-B2}
		\end{subfigure}
        \begin{subfigure}[t]{0.28\columnwidth}
			\centering
			\includegraphics[width=\textwidth]{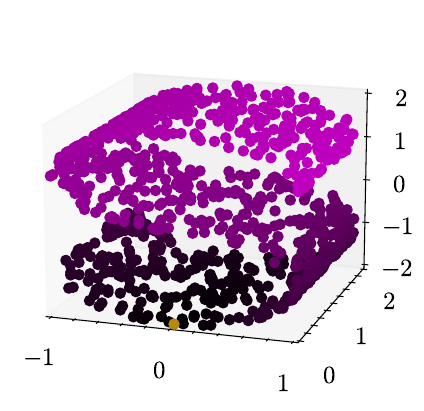}
			\caption{S Curve.}
			\label{perfil-B2}
		\end{subfigure}
		\begin{subfigure}[t]{0.28\columnwidth}
			\centering
			\includegraphics[width=\textwidth]{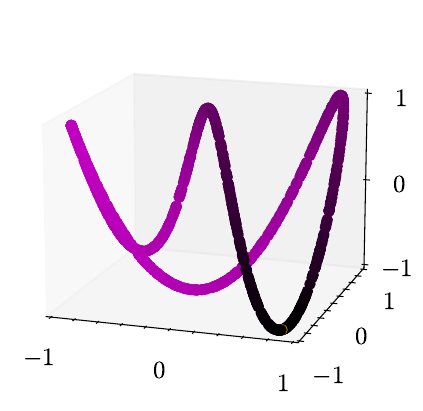}
			\caption{Helix.}
			\label{imagen-B2}
		\end{subfigure}
		\caption{Diffusion distances in $t$ steps with respect to the yellow point for the set $\mathcal{D}_a$.}
		\label{fig:dists-a}
\end{figure}

In Figures~\ref{fig:swiss_roll-test}, \ref{fig:s_curve-test}, \ref{fig:helix-test}, \ref{fig:mnist-test} and \ref{fig:phoneme-test}, the embeddings obtained by the three methods studied for the subset $\mathcal{D}_b$ (test set) of Swiss Roll, S Curve, Helix, MNIST and Phoneme, respectively, are collected. In the case of Swiss Roll, S Curve and Helix, it can be seen that the choice of $q$ has been adequate to correctly unfold the manifold, although for Swiss Roll and S Curve it has not been possible to find a value of $q$ that respects the width of the respective manifolds, instead of collapsing it to one dimension, as in Figures~\ref{fig:swiss_roll-test} and \ref{fig:s_curve-test}. The reason why these manifolds are projected as curves rather than surfaces is that the connectivity differences with other nodes between points distributed across these manifolds are negligible compared to the connectivity differences between points along them. Consequently, the diffusion distances between points distributed across these manifolds are much smaller than those between points distributed along them. On the other hand, both the Nystr\"om method and DDM seem to replicate quite faithfully the projection given by DM in all data sets, although the embedding of DDM presents, in general, a rotation and/or a reflection already contemplated with respect to that of DM, and also some more noise, due to a lack of total convergence to the global minimum of the cost function, something unavoidable in deep learning methods.

\begin{figure}[htbp]
		\centering
		\captionsetup{justification=centering}
            \begin{subfigure}[t]{0.19\columnwidth}
			\centering
			\includegraphics[width=\textwidth]{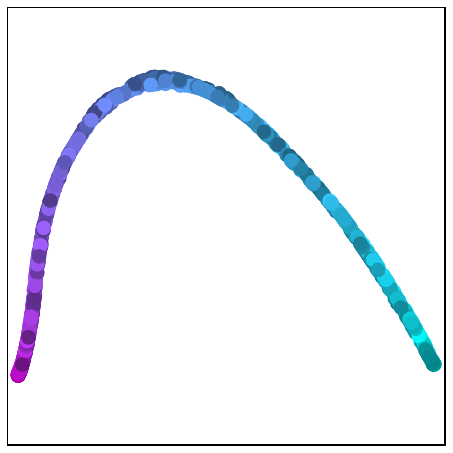}
			\caption{Nystr\"om.}
			\label{perfil-B2}
		\end{subfigure}
		\begin{subfigure}[t]{0.19\columnwidth}
			\centering
			\includegraphics[width=\textwidth]{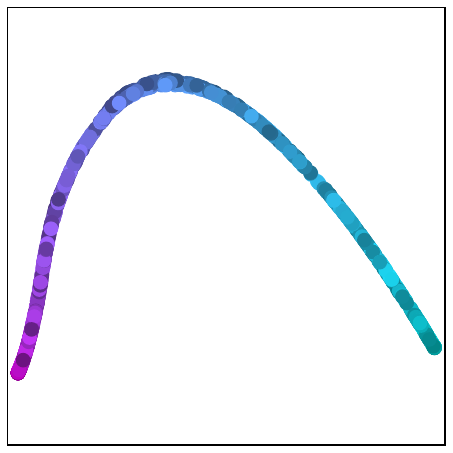}
			\caption{DM.}
			\label{imagen-B2}
		\end{subfigure}
		\begin{subfigure}[t]{0.19\columnwidth}
			\centering
			\includegraphics[width=\textwidth]{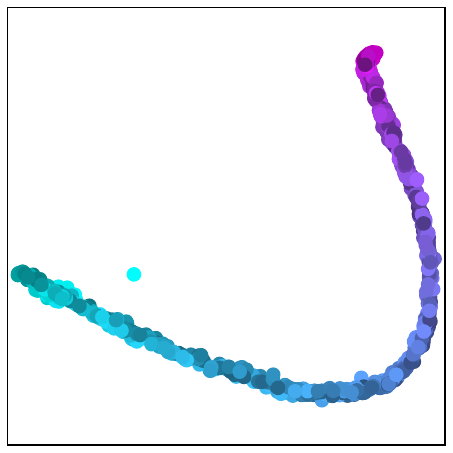}
			\caption{DDM.}
			\label{perfil-B2}
		\end{subfigure}
		\caption{Embedding of $\mathcal{D}_b$ for Swiss Roll.}
		\label{fig:swiss_roll-test}
\end{figure}

\begin{figure}[htbp]
		\centering
		\captionsetup{justification=centering}
            \begin{subfigure}[t]{0.19\columnwidth}
			\centering
			\includegraphics[width=\textwidth]{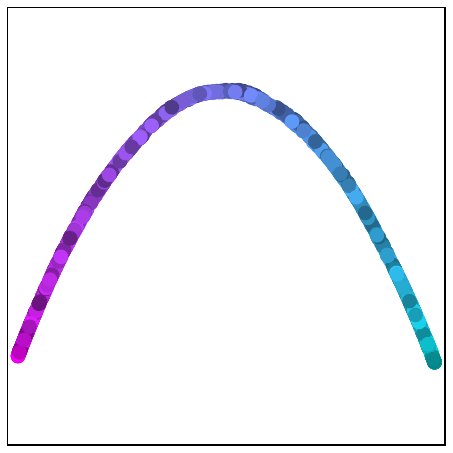}
			\caption{Nystr\"om.}
			\label{perfil-B2}
		\end{subfigure}
		\begin{subfigure}[t]{0.19\columnwidth}
			\centering
			\includegraphics[width=\textwidth]{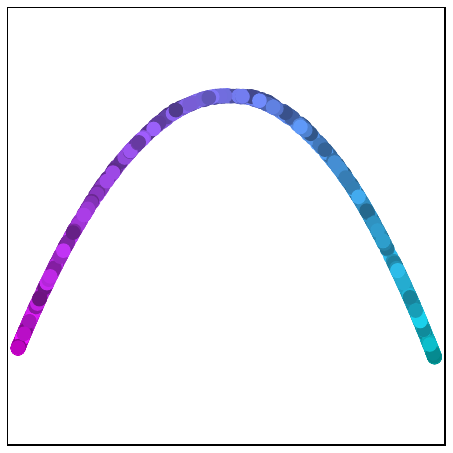}
			\caption{DM.}
			\label{imagen-B2}
		\end{subfigure}
		\begin{subfigure}[t]{0.19\columnwidth}
			\centering
			\includegraphics[width=\textwidth]{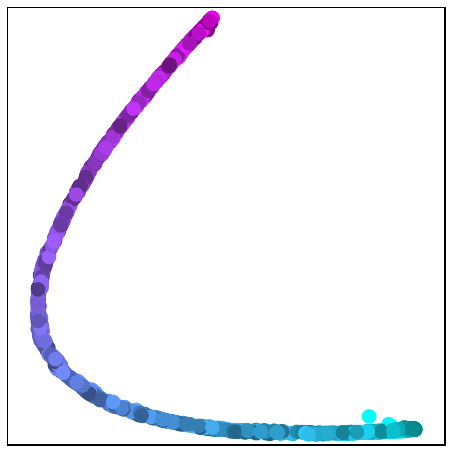}
			\caption{DDM.}
			\label{perfil-B2}
		\end{subfigure}
		\caption{Embedding of $\mathcal{D}_b$ for S Curve.}
		\label{fig:s_curve-test}
\end{figure}

\begin{figure}[htbp]
		\centering
		\captionsetup{justification=centering}
            \begin{subfigure}[t]{0.19\columnwidth}
			\centering
			\includegraphics[width=\textwidth]{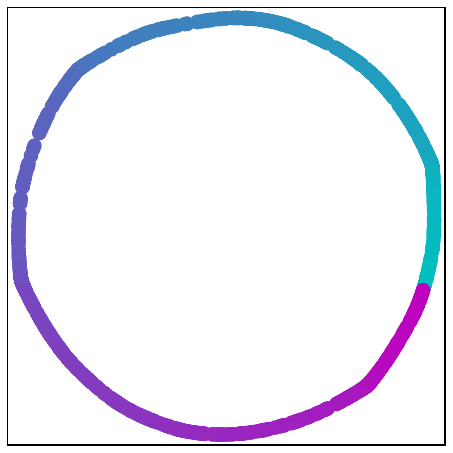}
			\caption{Nystr\"om.}
			\label{perfil-B2}
		\end{subfigure}
		\begin{subfigure}[t]{0.19\columnwidth}
			\centering
			\includegraphics[width=\textwidth]{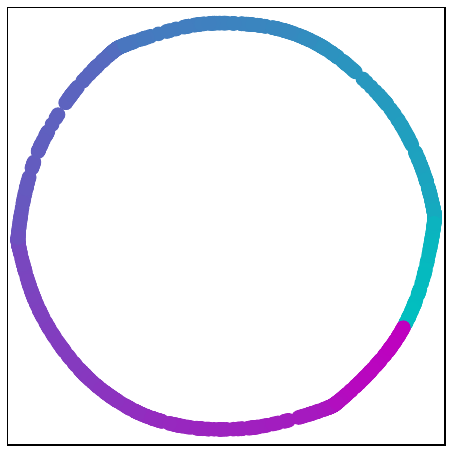}
			\caption{DM.}
			\label{imagen-B2}
		\end{subfigure}
		\begin{subfigure}[t]{0.19\columnwidth}
			\centering
			\includegraphics[width=\textwidth]{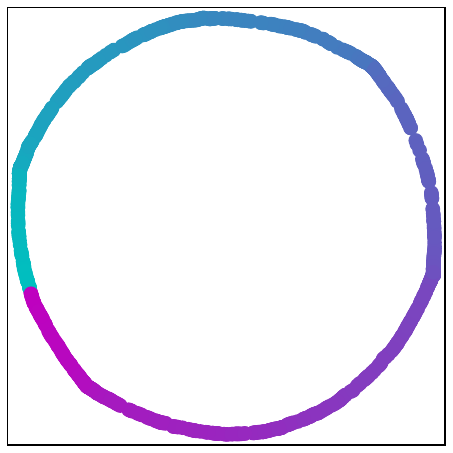}
			\caption{DDM.}
			\label{perfil-B2}
		\end{subfigure}
		\caption{Embedding of $\mathcal{D}_b$ for Helix.}
		\label{fig:helix-test}
\end{figure}

\begin{figure}[htbp]
		\centering
		\captionsetup{justification=centering}
            \begin{subfigure}[t]{0.19\columnwidth}
			\centering
			\includegraphics[width=\textwidth]{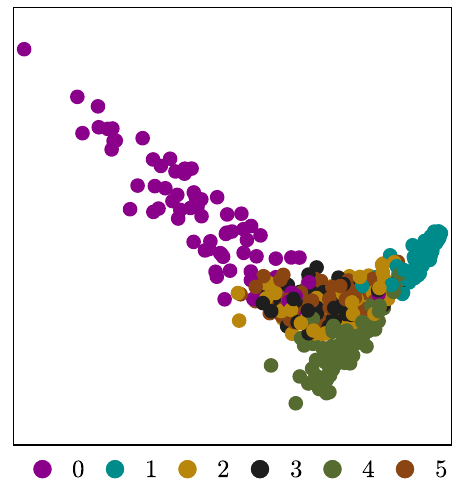}
			\caption{Nystr\"om.}
			\label{perfil-B2}
		\end{subfigure}
		\begin{subfigure}[t]{0.19\columnwidth}
			\centering
			\includegraphics[width=\textwidth]{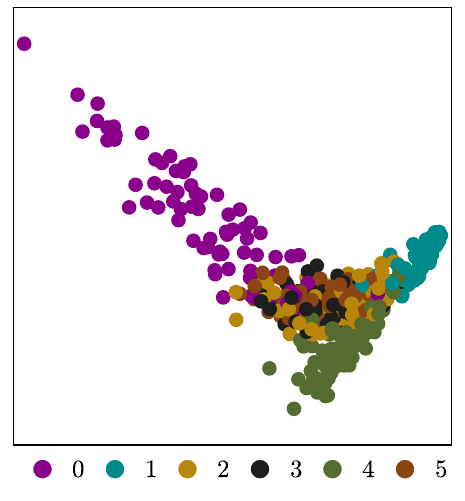}
			\caption{DM.}
			\label{imagen-B2}
		\end{subfigure}
		\begin{subfigure}[t]{0.19\columnwidth}
			\centering
			\includegraphics[width=\textwidth]{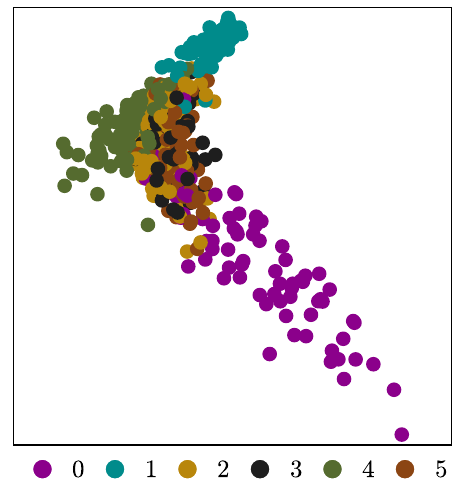}
			\caption{DDM.}
			\label{perfil-B2}
		\end{subfigure}
		\caption{Embedding of $\mathcal{D}_b$ for MNIST ($d=2$).}
		\label{fig:mnist-test}
\end{figure}

\begin{figure}[htbp]
		\centering
		\captionsetup{justification=centering}
            \begin{subfigure}[t]{0.19\columnwidth}
			\centering
			\includegraphics[width=\textwidth]{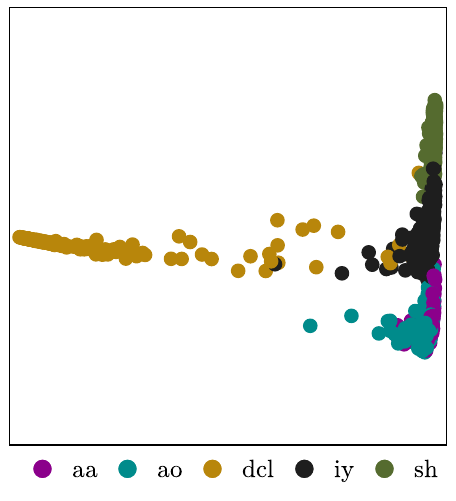}
			\caption{Nystr\"om.}
			\label{perfil-B2}
		\end{subfigure}
		\begin{subfigure}[t]{0.19\columnwidth}
			\centering
			\includegraphics[width=\textwidth]{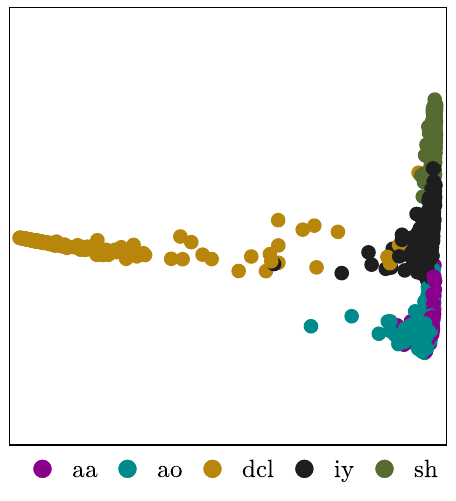}
			\caption{DM.}
			\label{imagen-B2}
		\end{subfigure}
		\begin{subfigure}[t]{0.19\columnwidth}
			\centering
			\includegraphics[width=\textwidth]{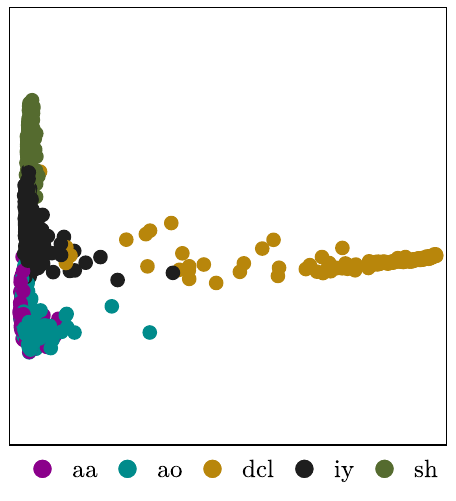}
			\caption{DDM.}
			\label{perfil-B2}
		\end{subfigure}
		\caption{Embedding of $\mathcal{D}_b$ for Phoneme.}
		\label{fig:phoneme-test}
\end{figure}

Finally, Table~\ref{tab:mre} shows the \SI{95}{\percent} confidence intervals using bootstrap percentiles for the mean relative error (MRE). It is clear that the Nystr\"om method replicates the results of the original DM algorithm better than DDM. However, the results of DDM are also acceptable. This is demonstrated by the results in Figures~\ref{fig:swiss_roll-test}, \ref{fig:s_curve-test}, \ref{fig:helix-test}, \ref{fig:mnist-test} and \ref{fig:phoneme-test}. In fact, as shown in Figure~\ref{fig:mre_by_decile} in the Appendix~\ref{sec:mre_by_decile}, the largest contribution to the MRE of DDM is by far found in the first decile of distances in the embedding, where the distances between points are so small that a slight deviation results in a very high value of the relative error. This means that, overall, DDM manages to replicate the embedding of DM considerably well for medium and large scales, but it shows significant differences in the distances between nearest neighbors, although these are imperceptible to the naked eye in most cases. For this reason, data sets such as Swiss Roll or MNIST present a relatively high MRE in the case of DDM, even though their embeddings are visually very similar to those of DM and the Nystr\"om method.

\begin{table}[htbp]
	\centering
        \captionsetup{justification=centering}
        \caption{95 \% confidence intervals using bootstrap percentiles for the mean relative error (MRE).}
	\label{tab:mre}
        \captionsetup{justification=centering}
	\begin{subtable}[t]{0.49\columnwidth}
		\centering
        \caption{Over $\mathcal{D}_a$.}
		\begin{tabular}{l c c}
			\toprule
			\textbf{Data set} & \textbf{Nystr\"om} & \textbf{DDM}\\
			\midrule
			Swiss Roll & $(10.5, 11.3)\,\%$  & $(15.0, 16.0)\,\%$ \\
                S Curve & $(9.2, 10.0)\,\%$  & $(9.4, 10.2)\,\%$ \\
			Helix & $(1.9, 1.9)\,\%$  & $(2.0, 2.1)\,\%$ \\
			MNIST ($d=2$) & $(2.5, 2.5)\,\%$  & $(15.5, 15.5)\,\%$ \\
			MNIST ($d=6$) & $(0.6, 0.6)\,\%$  & $(7.6, 7.6)\,\%$ \\
			Phoneme & $(0.8, 0.8)\,\%$  & $(6.2, 6.2)\,\%$ \\
			\bottomrule
		\end{tabular}
        \end{subtable}
        \mbox{}
        \begin{subtable}[t]{0.49\columnwidth}
		\centering
        \caption{Over $\mathcal{D}_b$.}
		\begin{tabular}{l c c}
			\toprule
			\textbf{Data set} & \textbf{Nystr\"om} & \textbf{DDM}\\
			\midrule
			Swiss Roll & $(10.8, 11.5)\,\%$  & $(20.2, 21.9)\,\%$ \\
                S Curve & $(9.5, 10.1)\,\%$  & $(10.7, 11.7)\,\%$ \\
			Helix & $(1.9, 1.9)\,\%$  & $(2.1, 2.1)\,\%$ \\
			MNIST ($d=2$) & $(2.7, 2.8)\,\%$  & $(17.2, 17.7)\,\%$ \\
			MNIST ($d=6$) & $(1.2, 1.2)\,\%$  & $(8.4, 8.5)\,\%$ \\
			Phoneme & $(0.8, 0.8)\,\%$  & $(6.1, 6.3)\,\%$ \\
			\bottomrule
		\end{tabular}
        \end{subtable}
\end{table}
\section{Conclusions}\label{sec:conclusions}

In this work, we propose to alleviate the problems of Diffusion Maps (DM), related to its computational cost and its inability to process data outside the initial set, by resorting to deep learning, as an alternative to the classic Nystr\"om method. To this end, we offer a new formulation of the DM embedding as a solution to a certain unconstrained minimization problem and, based on it, a cost function with which to train a neural network to obtain the DM embedding, for points both inside and outside the training sample and without the need to perform any spectral decomposition. The resulting method has been called Deep Diffusion Maps (DDM). Furthermore, the capabilities of this approach on different data sets, both real and synthetic, have been compared against those of DM and the Nystr\"om method.

While Nystr\"om's method appears to more closely replicate the results of DM from a quantitative point of view, DDM has proven capable of generating results almost identical to those of both methods at first glance, that is, from a qualitative point of view. The discrepancies observed in the mean relative error are concentrated in the smallest distances between points, between the nearest neighbors, which are almost imperceptible visually, but which result in very high relative error values. It is possible that, with somewhat more sophisticated training architectures and strategies than those employed in this work, the results in cases such as Swiss Roll and MNIST could be improved and MRE values closer to those of the Nystr\"om method could be obtained. In any case, the results are considerably good, and the main advantage of DDM lies in its potential gain in terms of computational cost~\cite{jansen_scalable_2017}, especially in the inference phase, that is, on data outside the training set. DDM could even become an interesting option from a training cost perspective: in the case of Nystr\"om, the more training data available, the longer the training time, while in the case of DDM, a larger number of data can reduce the number of epochs required and result in approximately the same training time. It should also be noted that the time advantage offered by neural methods is drastically increased when using specialized hardware such as GPUs or TPUs. Finally, the DDM cost function allows for easy integration of DM as a dimensionality reduction method with other deep learning methods, perhaps as a regularization term, which may be interesting for certain applications.

Regarding the limitations of DDM, since it is an iterative method, it can be more difficult to tune or train than DM or the Nystr\"om method. In the case of DDM, the training result depends on the initialization of the parameters and a greater number of hyperparameters, such as the architecture, batch size, learning rate, etc. On the other hand, it has been observed that, in some cases, when the gradient of the eigenvalue curve is very small (almost flat), as in Helix or Swiss Roll when $t=1$, DDM has difficulties converging to the global minimum of the cost function. However, it seems that this can be alleviated by adjusting certain hyperparameters, for example, increasing the batch size and reducing the learning rate. Despite all this, it can be concluded that DDM has proven to be a good alternative to the original DM algorithm and the Nystr\"om method in contexts where their computational cost is too high.
\FloatBarrier
\appendix
\section{General Nystr\"om method}\label{sec:general_nystrom}

\subsection{Eigenvectors and eigenvalues of one sample from another}
Let $\mathcal{H}_p$ be the Hilbert space of functions $f:\R^D\to\R$ such that
\begin{equation}
	\int f^2(\vb{x}) p(\vb{x})\dd{\vb{x}} < +\infty,
\end{equation}
where $p:\R^D\to\R$ is a certain probability density function. Let $\mathcal{K}:\R^D\times\R^D\to\R$ be a kernel function and let $\Bqty{u_i:\R^D\to\R}_{i}$ be the eigenfunctions of the operator it defines, i.e.,
\begin{equation}
	\int \mathcal{K}(\vb{x},\vb{y})u_i(\vb{y})p(\vb{y})\dd{\vb{y}} = \lambda_i u_i(\vb{x}).
\end{equation}
Since $\mathcal{K}$ is real and symmetric, the eigenfunctions are orthonormal:
\begin{equation}
	\int u_i(\vb{x})u_j(\vb{x})p(\vb{x})\dd{\vb{x}} = \delta_{i, j}.
\end{equation}
These integrals can be viewed as expected values according to the distribution defined by $p$. Therefore, they can be approximated from a sample, $\Bqty{\vb{x}_{k}}_{k=1}^M$, i.i.d. of $p$ by replacing the expected value by the empirical mean:
\begin{gather}\label{eq:media_empirica_nys}
	\dfrac{1}{M}\sum_{k=1}^M\mathcal{K}\pqty{\vb{x}, \vb{x}_{k}}u_i\pqty{\vb{x}_{k}}\approx\lambda_iu_i(\vb{x}),\\
	\label{eq:ortonormalidad_empirica_nys}
	\dfrac{1}{M}\sum_{k=1}^M u_i\pqty{\vb{x}_{k}}u_j\pqty{\vb{x}_{k}}\approx\delta_{ij}.
\end{gather}

Now consider the matrix of dimensions $M\times M$ given by
\begin{equation}
	K^{\aqty{M}}_{k, l} = \mathcal{K}\pqty{\vb{x}_{k},\vb{x}_{l}},\qquad k, l = 1, \ldots, M.
\end{equation}
Because it is real and symmetric, it has a set of $M$ eigenvalues
and orthonormal eigenvectors,
both real, given by
\begin{equation}
	\vb{K}^{\aqty{M}} \vb{u}^{\aqty{M}}_i = \lambda_i^{\aqty{M}}\vb{u}^{\aqty{M}}_i,
\end{equation}
or, in terms of components, by
\begin{equation}\label{eq:ec_autovalores_K_nys}
	\sum_{k=1}^M K^{\aqty{M}}_{j, k} \pqty{\vb{u}^{\aqty{M}}_i}_k = \lambda_i^{\aqty{M}}\pqty{\vb{u}^{\aqty{M}}_i}_j,
\end{equation}
for $i, j = 1,\ldots, M$. 
On the other hand, replacing $\vb{x}$ by $\vb{x}_{j}$ in \eqref{eq:media_empirica_nys}, we can write
\begin{equation}
	\dfrac{1}{M}\sum_{k=1}^MK^{\aqty{M}}_{j, k}u_i\pqty{\vb{x}_{k}}\approx\lambda_iu_i\pqty{\vb{x}_{j}}.
\end{equation}
Comparing this expression with \eqref{eq:ec_autovalores_K_nys}, and taking into account \eqref{eq:ortonormalidad_empirica_nys}, we have that
\begin{gather}
	u_i\pqty{\vb{x}_{j}}\approx\sqrt{M} \pqty{\vb{u}^{\aqty{M}}_i}_j,\label{eq:nys_aprox_1}\\
    \lambda_i\approx \dfrac{1}{M}\lambda_i^{\aqty{M}}.\label{eq:nys_aprox_1_eigenvalues}
\end{gather}
Substituting this back into \eqref{eq:media_empirica_nys}, we can solve for
\begin{equation}\label{eq:nys_aprox_2}
	u_i\pqty{\vb{x}}\approx\dfrac{\sqrt{M}}{\lambda_i^{\aqty{M}}}\sum_{k=1}^M\mathcal{K}\pqty{\vb{x}, \vb{x}_{k}}\pqty{\vb{u}^{\aqty{M}}_i}_k.
\end{equation}

Suppose now that we have a second sample, $\Bqty{\tilde{\vb{x}}_{k}}_{k=1}^{N}$, of $p$ with $N$ elements. Then we can approximate the value of $u_i$ at each of the points in the second sample in two different ways: with \eqref{eq:nys_aprox_1}, using the second sample, and with \eqref{eq:nys_aprox_2}, using the first one. Equating the two approximations, we obtain
\begin{equation}\label{eq:nys_eigenvectors}
	\pqty{\tilde{\vb{u}}^{\aqty{N}}_i}_j \approx \dfrac{\sqrt{M/N}}{\lambda_i^{\aqty{M}}}\sum_{k=1}^M\mathcal{K}\pqty{\tilde{\vb{x}}_{j}, \vb{x}_{k}}\pqty{\vb{u}^{\aqty{M}}_i}_k,
\end{equation}
with $ 1\leq i\leq M$ and $1\leq j\leq N$. 
As for the eigenvalues, equating the approximations obtained using \eqref{eq:nys_aprox_1_eigenvalues} with each sample, we have that
\begin{equation}\label{eq:nys_eigenvalues}
	\tilde{\lambda}_i^{\aqty{N}}\approx\dfrac{N}{M}\lambda_i^{\aqty{M}}.
\end{equation}
According to~\cite{williams_using_2000}, when $N > M$, these approximations reduce the memory cost of computing $M$ eigenvectors and eigenvalues of $\vb{K}^{\aqty{N}}$ from $\order{N^2}$ to $\order{NM}$ and the computational complexity from $\order{N^3}$ to $\order{M^3 + NM^2} = \order{NM^2}$.

\subsection{Eigenvectors and eigenvalues of an extended sample}

The above approximations can also be seen as a way to extend the eigenvalues and eigenvectors of a kernel matrix when the initial data sample is extended with new points.
To see this, consider that we have two samples, $\Bqty{\vb{x}_{k}}_{k=1}^{M}$ and $\Bqty{\vb{x}_{k}}_{k=M+1}^{N+M}$, from the same distribution, $p$, and we want to extend the eigenvectors and eigenvalues of the first sample to the union of the two samples, $\Bqty{\vb{x}_{k}}_{k=1}^{N+M}$. Then, applying \eqref{eq:nys_eigenvectors} and \eqref{eq:nys_eigenvalues}, we obtain
\begin{gather}
	\pqty{\vb{u}^{\aqty{N+M}}_i}_j \approx \dfrac{\sqrt{M/(N+M)}}{\lambda_i^{\aqty{M}}}\sum_{k=1}^M\mathcal{K}\pqty{\vb{x}_{j}, \vb{x}_{k}}\pqty{\vb{u}^{\aqty{M}}_i}_k,\\
	\lambda_i^{\aqty{N+M}}\approx\dfrac{N + M}{M}\lambda_i^{\aqty{M}},
\end{gather}
for $i = 1,\ldots, M$ and $j = 1, \ldots, N+M$. When $N/M\to 0$, we have that
\begin{gather}
	\pqty{\vb{u}^{\aqty{N+M}}_i}_j \approx \dfrac{1}{\lambda_i^{\aqty{M}}}\sum_{k=1}^M\mathcal{K}\pqty{\vb{x}_{j}, \vb{x}_{k}}\pqty{\vb{u}^{\aqty{M}}_i}_k,\\
	\lambda_i^{\aqty{N+M}}\approx\lambda_i^{\aqty{M}}.
\end{gather}
This is precisely the approach proposed in~\cite{bengio_out--sample_2003}.
\section{Eigenvalue and likelihood curves}

In Figure~\ref{fig:eigenvalues-loglikelihood}, for the subset $\mathcal{D}_a$ of each data set, the curve with the first 25 eigenvalues (excluding the eigenvalue equal to $1$) ordered from largest to smallest and the curve with the logarithm of the likelihood as a function of the dimension, $d$, of the embedding are collected, as in \cite{zhu_automatic_2006}.

\begin{figure}[H]
		\centering
		\captionsetup{justification=centering}
            \begin{subfigure}[t]{0.45\columnwidth}
			\centering
			\includegraphics[width=\textwidth]{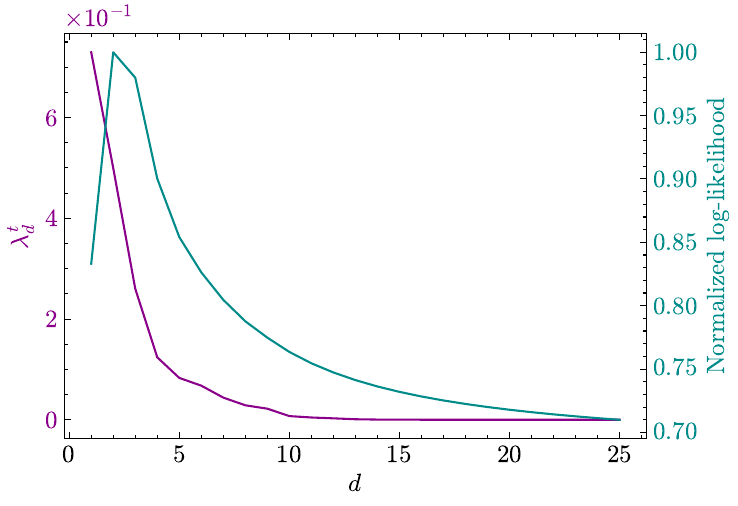}
			\caption{Swiss Roll.}
			\label{perfil-B2}
		\end{subfigure}
		\begin{subfigure}[t]{0.45\columnwidth}
			\centering
			\includegraphics[width=\textwidth]{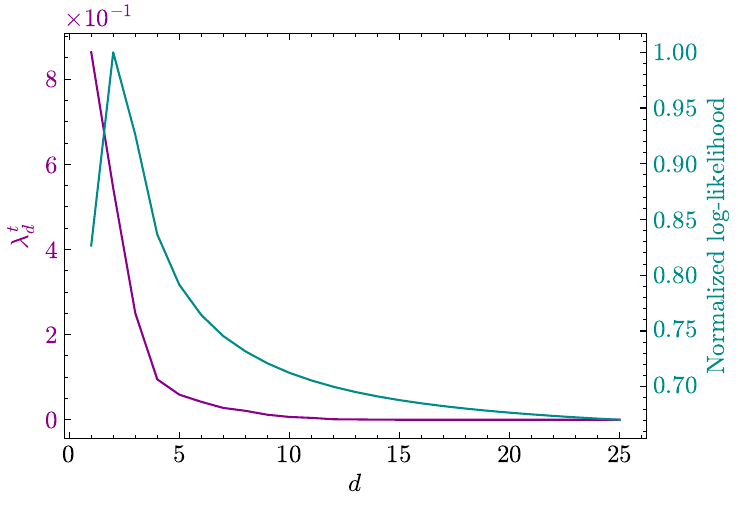}
			\caption{S Curve.}
			\label{imagen-B2}
		\end{subfigure}
        \begin{subfigure}[t]{0.45\columnwidth}
			\centering
			\includegraphics[width=\textwidth]{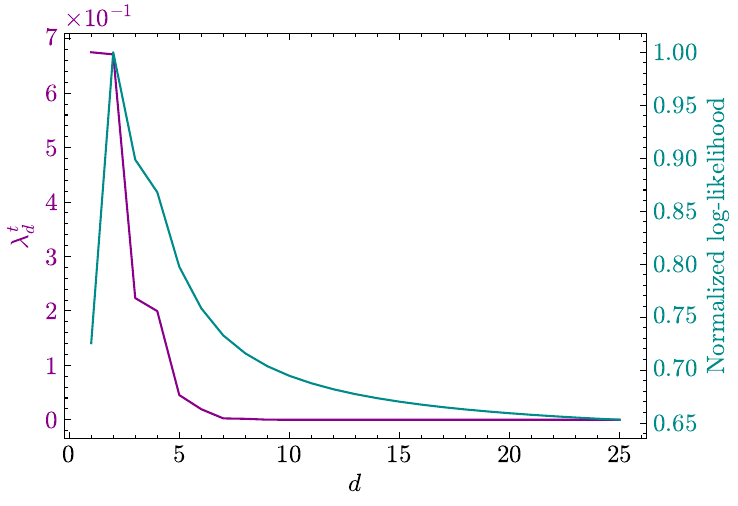}
			\caption{Helix.}
			\label{imagen-B2}
		\end{subfigure}
		\begin{subfigure}[t]{0.45\columnwidth}
			\centering
			\includegraphics[width=\textwidth]{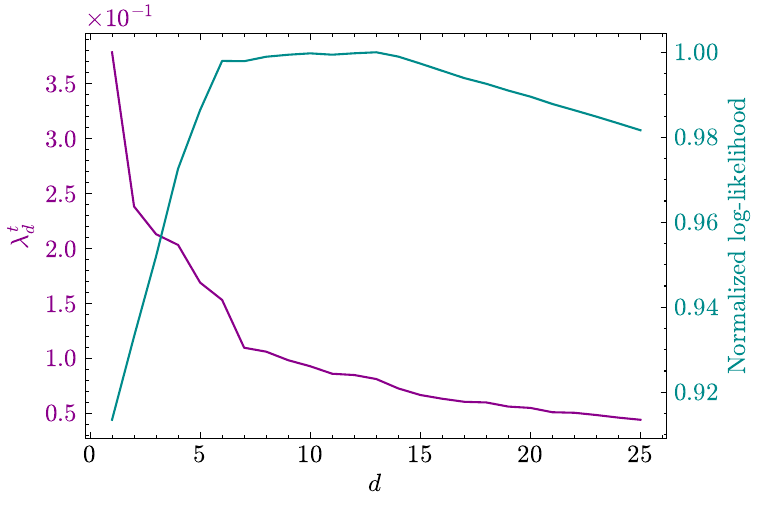}
			\caption{MNIST.}
			\label{perfil-B2}
		\end{subfigure}
        \caption{Eigenvalue and log-likelihood curves for $\mathcal{D}_a$.}
		\label{fig:eigenvalues-loglikelihood}
\end{figure}
\begin{figure}[H]
\ContinuedFloat
		\centering
		\captionsetup{justification=centering}
		\begin{subfigure}[t]{0.45\columnwidth}
			\centering
			\includegraphics[width=\textwidth]{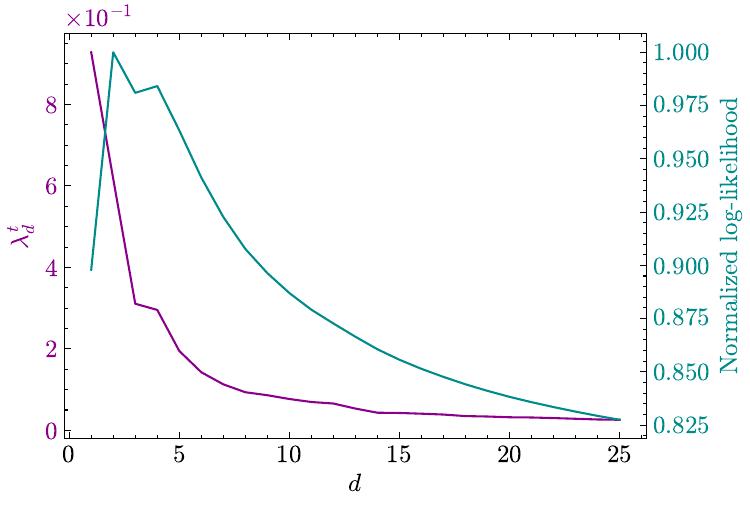}
			\caption{Phoneme.}
			\label{imagen-B3}
		\end{subfigure}
		\caption{Eigenvalue and log-likelihood curves for $\mathcal{D}_a$.}
\end{figure}
\section{MRE by decile of distances}\label{sec:mre_by_decile}

In Figure~\ref{fig:mre_by_decile}, the MRE value of Deep Diffusion Maps and the Nyström method is represented for each decile of (Euclidean) distances in the Diffusion Maps embedding.

\begin{figure}[H]
		\centering
		\captionsetup{justification=centering}
            \begin{subfigure}[t]{0.45\columnwidth}
			\centering
			\includegraphics[width=\textwidth]{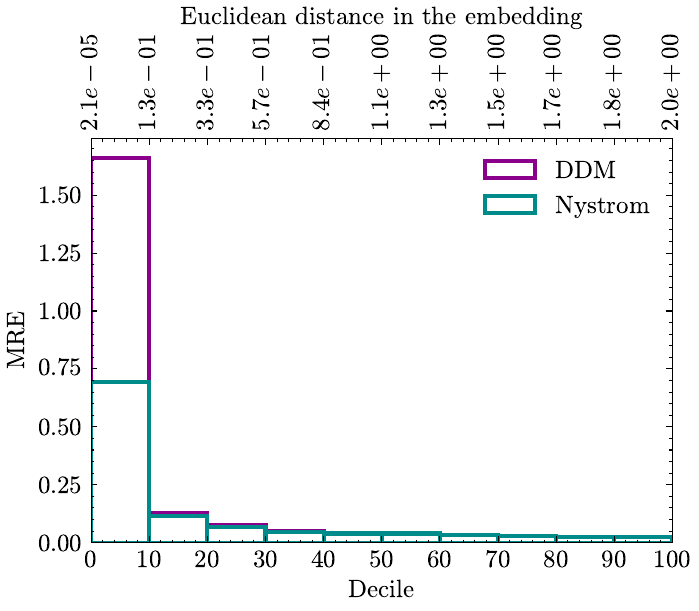}
			\caption{Swiss Roll.}
			\label{perfil-B2}
		\end{subfigure}
        \begin{subfigure}[t]{0.45\columnwidth}
			\centering
			\includegraphics[width=\textwidth]{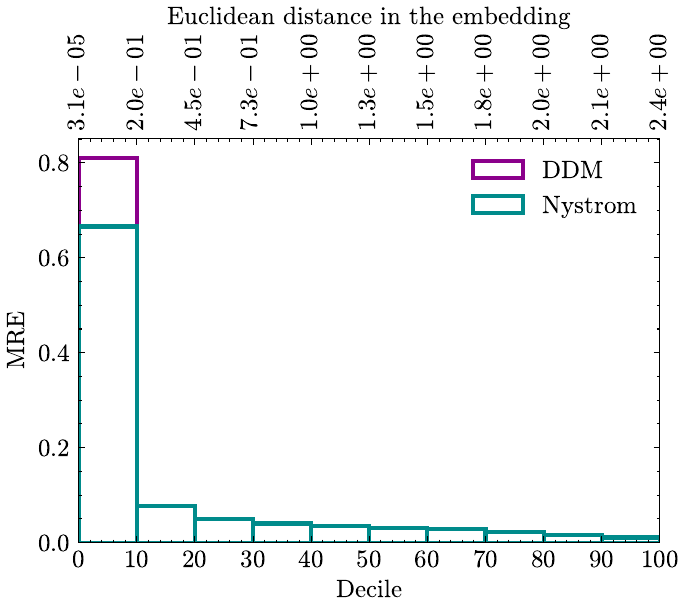}
			\caption{S Curve.}
			\label{perfil-B2}
		\end{subfigure}
		\begin{subfigure}[t]{0.45\columnwidth}
			\centering
			\includegraphics[width=\textwidth]{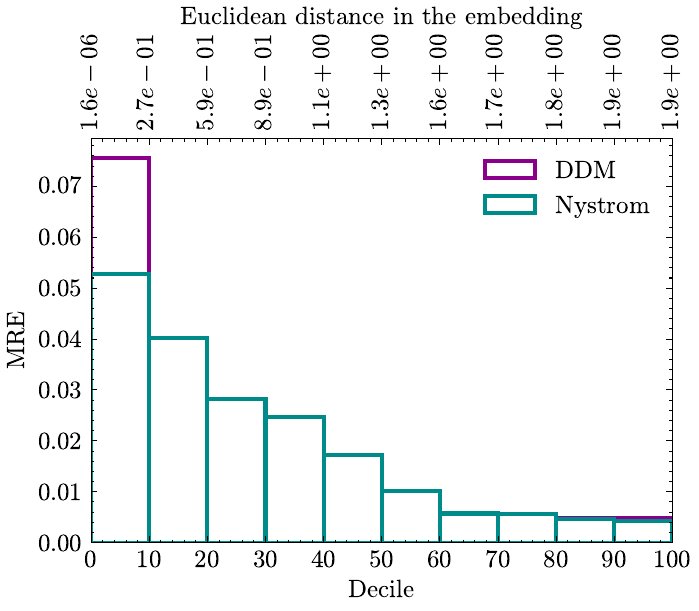}
			\caption{Helix.}
			\label{imagen-B2}
		\end{subfigure}
    \begin{subfigure}[t]{0.45\columnwidth}
        \centering
        \includegraphics[width=\textwidth]{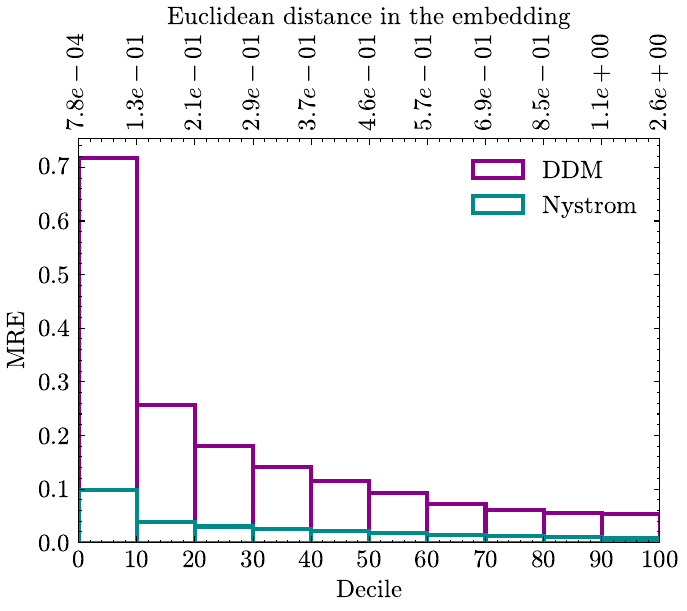}
        \caption{MNIST ($d=2$).}
        \label{perfil-B2}
    \end{subfigure}
    \caption{MRE per decile of Euclidean distances in the Diffusion Maps embedding for $\mathcal{D}_b$.}
		\label{fig:mre_by_decile}
\end{figure}
\begin{figure}[H]
\ContinuedFloat
    \centering
    \captionsetup{justification=centering}
    \begin{subfigure}[t]{0.45\columnwidth}
			\centering
			\includegraphics[width=\textwidth]{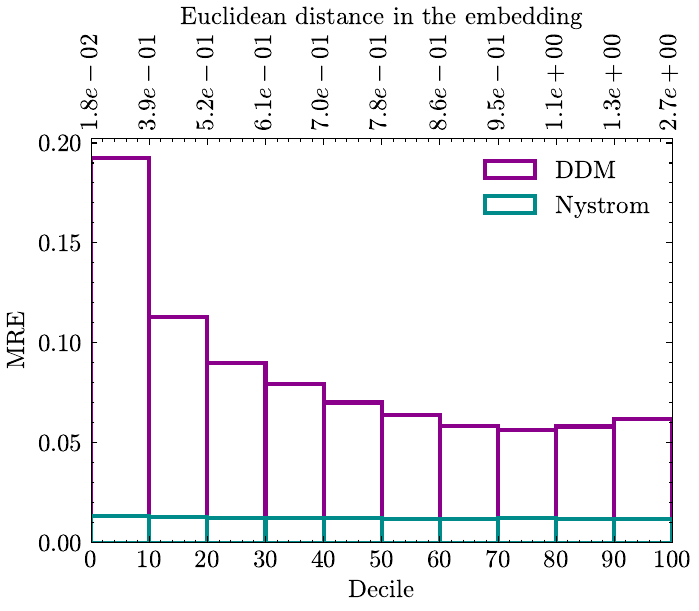}
			\caption{MNIST ($d=6$).}
			\label{perfil-B2}
		\end{subfigure}
		\begin{subfigure}[t]{0.45\columnwidth}
			\centering
			\includegraphics[width=\textwidth]{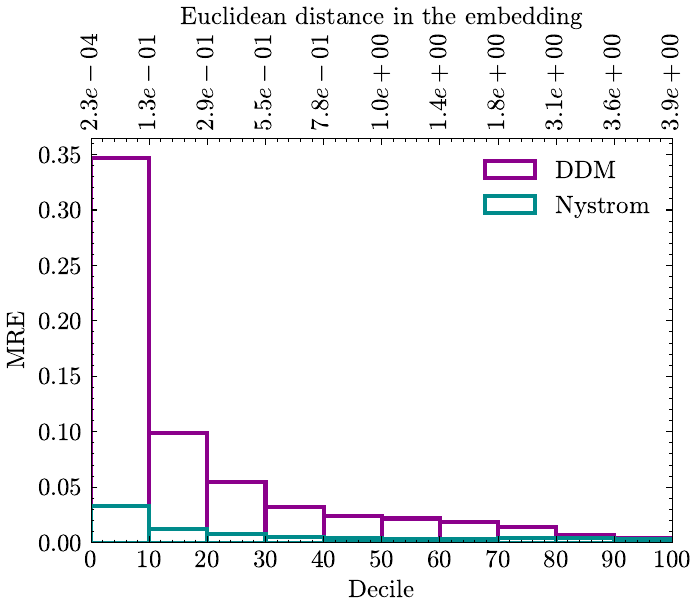}
			\caption{Phoneme.}
			\label{imagen-B3}
		\end{subfigure}
		\caption{MRE per decile of Euclidean distances in the Diffusion Maps embedding for $\mathcal{D}_b$.}
\end{figure}
\section*{Acknowledgements}
The authors acknowledge financial support from project PID2022-139856NB-I00 funded by MCIN/ AEI / 10.13039/501100011033 / FEDER, UE and project IDEA-CM (TEC-2024/COM-89) from the Autonomous Community of Madrid and from the ELLIS Unit Madrid.  The authors acknowledge computational support from the Centro de Computación Científica-Universidad Autónoma de Madrid (CCC-UAM).
\FloatBarrier
\bibliographystyle{plain}
\bibliography{biblio}

\end{document}